\documentclass{article}

\PassOptionsToPackage{numbers, compress}{natbib}
 \usepackage[preprint]{neurips_2026}

\usepackage[utf8]{inputenc} 
\usepackage[T1]{fontenc}    
\usepackage{hyperref}       
\usepackage{url}            
\usepackage{booktabs}       
\usepackage{amsfonts}       
\usepackage{nicefrac}       
\usepackage{microtype}      
\usepackage{xcolor}         

\usepackage{graphicx}
\usepackage{subcaption}

\usepackage{algorithm}
\usepackage{algpseudocode}

\usepackage{amsmath}
\usepackage{amssymb}
\usepackage{mathtools}
\usepackage{amsthm}
\usepackage{multirow}

\usepackage[capitalize,noabbrev]{cleveref}

\theoremstyle{plain}
\newtheorem{theorem}{Theorem}[section]
\newtheorem{proposition}[theorem]{Proposition}

\theoremstyle{definition}
\newtheorem{definition}[theorem]{Definition}

\theoremstyle{remark}

\usepackage{wrapfig}

\title{Path Integral Value Matching for Linear Quadratic Stochastic Optimal Control}

\author{
  \textbf{Bangyan Liao},
  \textbf{Chenglei Yu},
  \textbf{Yuchen Yang},
  \textbf{Chuanrui Wang}, \\
  \textbf{Zhisheng Song}, 
  \textbf{Peidong Liu},
  \textbf{Tailin Wu}$^{\dagger}$  \\[0.5em]
  Westlake University \\
 School of Engineering \\[0.5em]
  \textcolor{purple}{\textbf{{Project Page}: \url{https://github.com/bangyan101/PIVM/}}
}
}

\begin{document}

\maketitle

\begin{abstract}
Linear Quadratic Stochastic Optimal Control (LQ-SOC) establishes a fundamental framework for steering noisy dynamical systems and has recently gained renewed interest in the machine learning community. However, current state-of-the-art policy-based methods suffer from prohibitive computational costs and instability due to their heavy reliance on {full-trajectory simulation}. To overcome these limitations, we propose a paradigm shift toward a value-based approach by revisiting Path Integral Control (PIC). Although standard PIC suffers from the same high-variance bottleneck as policy-based methods, we discover that by \textit{truncating and marginalizing} the original path integral formulation, we can derive a temporal recursive form of the value function. Building upon this theoretical foundation, we propose the \textbf{Path Integral Value Matching (PI-VM)} algorithm. Specifically, we employ temporal-difference learning to approximate the recursive value dynamics, and further integrate the Girsanov theorem with experience replay to enable off-policy training. We benchmark PI-VM against SOTA policy-based methods across various SOC benchmarks and sampling tasks. Empirical results demonstrate that PI-VM matches SOTA precision with an order-of-magnitude efficiency gain in low-dimensional settings, while effectively mitigating mode collapse in high-dimensional scenarios. Consequently, PI-VM offers a scalable solution for solving complex SOC problems. 

\end{abstract}
\section{Introduction}
Stochastic Optimal Control (SOC)~\cite{bellman1966dynamic, fleming2012deterministic} establishes a fundamental framework for steering noisy dynamical systems toward high-reward regions. While dating back several decades, the Linear-Quadratic Stochastic Optimal Control (LQ-SOC) formulation~\cite{domingo2024stochastic} has recently attracted renewed interest within the machine learning community. This is driven by its profound theoretical connections to generative modeling (via diffusion models~\cite{domingo-enrich2025adjoint}), Optimal Transport (via Schrödinger Bridges~\cite{chen2021stochastic}), and energy-based model sampling~\cite{liu2025adjoint}.

Despite its theoretical significance, solving SOC problems remains computationally prohibitive. Current state-of-the-art approaches largely follow the policy-based paradigm, known as the Iterative Diffusion Optimization~\cite{domingo2024stochastic, domingo-enrich2025adjoint}. These methods optimize neural control policies by minimizing the divergence between parameterized path measures and optimal path measures. However, a critical bottleneck restricts their scalability: they rely heavily on on-policy full-trajectory simulation for gradient estimation. This dependence on long-horizon stochastic simulation leads to excessive computational costs and high-variance gradient estimates in high dimension~\cite{blessing2025trust, domingo-enrich2025adjoint}. Moreover, the variance of off-policy importance weights blows up when the dimension is high, which further destabilizes the off-policy training process~\cite{domingo2024taxonomy}.

To fundamentally overcome these limitations, we propose a paradigm shift toward a value-based approach by revisiting Path Integral Control (PIC)~\cite{todorov2006linearly, kappen2007introduction, kappen2005linear}. Unlike policy-based methods that search locally, PIC aims to solve the Hamilton-Jacobi-Bellman (HJB)~\cite{bellman1966dynamic} equation globally. However, naive PIC has historically remained practically intractable because—similar to policy-based methods—it inherently relies on high-variance full-trajectory simulation. Specifically, traditional PI methods require sampling complete trajectories from the current time $t$ to the terminal time to estimate the value function. This results in exploding variance in high-dimensional stochastic environments~\cite{thijssen2015path}.


Consequently, the field faces a dilemma: both dominant policy-based methods and classical value-based PI methods are bottlenecked by the high variance and computational cost of full-trajectory simulation. To bridge this gap, we truncate and marginalize the original path integral formulation, and derive a continuous-time recursive form of the value function. 

Motivated by this insight, we introduce the \textbf{Path Integral Value Matching (PI-VM)} algorithm. PI-VM employs Temporal Difference (TD) learning to approximate the recursive value dynamics, effectively replacing high-variance long-horizon integration with stable short-term bootstrapping. Furthermore, we integrate the Experience Replay buffer~\cite{zhang2017deeper} with the Girsanov theorem~\cite{oksendal2013stochastic} to support off-policy training. This combination allows for mathematically grounded trajectory reweighting, enabling flexible transitions between on-policy exploitation and off-policy  exploration.

Our main contributions are summarized as follows: \begin{itemize}
\item  We bypass the full-trajectory simulation by deriving a continuous-time recursive form of the value function, establishing the theoretical foundation for efficient and stable value training.
\item We propose a practical solver \textbf{PI-VM} that utilizes an off-policy TD loss to learn the value recursive dynamics with experience replay and the Girsanov theorem.
\item Experiments on SOC benchmarks and sampling tasks demonstrate that PI-VM outperforms existing baselines in terms of training speed, numerical stability, and control accuracy, offering a scalable paradigm for stochastic control.
\end{itemize}

\section{Preliminaries and Related Works}
\subsection{Preliminaries}
\label{label:Preliminaries}
In this section, we introduce the preliminary background for stochastic optimal control and path integral control. For a detailed introduction and survey, please refer to~\cite{fleming2012deterministic, hu2023recent}.

\paragraph{Linear Quadratic SOC.} 
We define the Linear Quadratic Stochastic Optimal Control (LQ-SOC) problem as:
\begin{equation}
\label{eq:lq_soc_obj}
\begin{aligned}
\min_{u \in \mathcal{U}} \mathbb{E}_{\mathbb{P}^u} &\left[  \int_{0}^{1}  \left( \frac{1}{2}\left\|u\left(X_{t}, t\right)\right\|^{2}+f\left(X_{t}, t\right)\right) \mathrm{d} t+g\left(X_{1}\right)\right], \\
\text { s.t. } \mathrm{d} X_{t}&=\left(b\left(X_{t}, t\right)+\sigma(t) u\left(X_{t}, t\right)\right) \mathrm{d} t+\sigma(t) \mathrm{d} B_{t}, \quad X_{0} \sim \rho_0.
\end{aligned}
\end{equation}
where $b \in C^{1,0}(\mathbb{R}^d \times [0, 1]; \mathbb{R}^d)$ is the drift function, $\sigma \in C([0, 1]; \mathbb{R}^{d \times m})$ 
denotes the diffusion coefficient, $\rho_0$ represents the source distribution, process $\{B_t\}_{t \in [0, 1]}$ is an $\mathcal{F}_t$-standard Brownian motion, $u \in C^{1,0}(\mathbb{R}^d \times [0, 1]; \mathbb{R}^d)$ is the control signal within the set of admissible control $\mathcal{U}$, $f \in C^{1,0}(\mathbb{R}^d \times [0, 1]; \mathbb{R})$ denotes the state cost (or running cost), and $g \in C^1(\mathbb{R}^d; \mathbb{R})$ represents the terminal cost. 

The state process $\{X_t\}_{t \in [0, 1]}$ is governed by the controlled SDE in \cref{eq:lq_soc_obj}, which induces the controlled path measure $\mathbb{P}^u$ on $\mathcal{C}([0, 1]; \mathbb{R}^d)$. For the uncontrolled SDE ($u=0$), it induces a reference path measure $\mathbb{P}^{0}$ on $\mathcal{C}([0, 1]; \mathbb{R}^d)$. We denote $\mathbb{P}_{[t, s]}^{0}$ as the restriction of this measure to the path segment $\{X_r\}_{r \in [t, s]}$. 

For simplicity, we define the work functional as: 
\begin{equation}
    \mathcal{W}(X, t)=\int_{t}^{1}  f\left(X_{s}, s\right) \mathrm{d} s+g\left(X_{1}\right).
\end{equation}
Although the LQ formulation in \cref{eq:lq_soc_obj} provides a structural simplification compared to the general SOC problem, obtaining the optimal solution remains a non-trival task. The optimization process is significantly impeded by the prohibitive cost of forward simulations, non-convex landscapes, and inherent stochasticity.
\paragraph{Cost functional and value function.} Given an initial condition $(x, t) \in \mathbb{R}^d \times [0,1]$, the expected future cost incurred by following a control $u$ is defined as
\begin{equation}
\begin{aligned}
J(u ; x, t) :=  \mathbb{E}_{ \mathbb{P}^u} \bigg[  \int_{t}^{1}  \frac{1}{2} \|u(X_{s}, s)\|^{2}  \mathrm{d} s
 +  \mathcal{W}(X, t) \bigg| X_{t} = x \bigg]
\end{aligned}
\end{equation}
Then, we can define the corresponding value function as the cost functional with optimal control $u^*$:
\begin{equation}
    V(x,t) := \min_{u \in \mathcal{U}} J(u;x,t)=J(u^*;x,t).
\end{equation}
Learning the value function is more tractable than direct optimization due to its smoother geometry and ability to leverage global information against non-convexity and noise. Moreover, in continuous-time RL, the value function is a more natural target than the Q-function, as the latter collapses to the former in the continuous limit~\cite{jia2023q}.


\paragraph{Path Integral Control.}
A key benefit of LQ-SOC is that the optimal value function satisfies a nonlinear partial differential equation known as the Hamilton–Jacobi–Bellman (HJB)~\cite{bellman1966dynamic} equation. Building on the HJB equation, the value function admits a closed-form representation known as the path integral control~\cite{todorov2006linearly, kappen2007introduction, kappen2005linear}. In particular, the value function can be expressed as
\begin{equation}
\label{eq:path_integral_form_of_v}
\begin{aligned}
&V(x,t) = -\log \mathbb{E}_{\mathbb{P}^0} \Bigg[ \exp \bigg( -\mathcal{W}(X, t) \bigg) \,\Big|\, X_t = x \Bigg],
\end{aligned}
\end{equation}

where $X_t$ follows the uncontrolled stochastic differential equation associated with the base drift. This can be proved by first applying the Cole-Hopf transform~\cite{evans2022partial} to convert the nonlinear HJB equation into a linear PDE, and then invoking the Feynman-Kac lemma~\cite{oksendal2013stochastic} to interpret the solution as a conditional expectation. 
Furthermore, the HJB equation not only characterizes the optimal value function but also yields the optimal control in closed form: 
\begin{equation}
    u^*(x,t) = - \sigma^T(t) \nabla V(x,t).
\end{equation} 
Nevertheless, practical implementation is hindered by the curse of dimensionality. High dimensions and noise induce exploding variance in the stochastic estimator, making the computational cost of Monte Carlo integration intractable.

\subsection{Related Works}

\paragraph{Policy-based SOC Solvers.} Current SOC solvers predominantly follow the \textit{Iterative Diffusion Optimization} paradigm, optimizing control policies by minimizing divergences (e.g., Cross-Entropy or Variance)~\cite{holdijk2023stochastic, nusken2021solving} or employing matching-based strategies like Adjoint Matching~\cite{domingo2024stochastic, domingo-enrich2025adjoint, domingo2024taxonomy}. While recent advances have incorporated trust-region techniques~\cite{blessing2025trust, von2025learning} and expanded into sampling tasks~\cite{liu2025adjoint, havens2025adjoint}, these methods rely heavily on on-policy forward simulation and stochastic back-propagation. Consequently, they suffer from significant computational overhead and high-variance gradient estimates, rendering them susceptible to instability and mode collapse in complex landscapes.

\paragraph{Value-based SOC Solvers.} Numerical methods for the HJB equation, such as finite difference schemes~\cite{bonnans2004fast, jensen2013convergence}, theoretically offer a global solution but are severely limited by the curse of dimensionality. Alternatively, Path Integral Control (PIC)~\cite{kappen2005linear, todorov2006linearly} applies the Feynman-Kac lemma to transform the PDE into a sampling problem, inspiring algorithms like MPPI~\cite{williams2017model} and Cross-Entropy methods~\cite{kappen2016adaptive}. However, standard PIC faces a critical bottleneck: the variance of the path integral estimator scales poorly with dimensionality~\cite{kazim2024recent}. In high-dimensional settings, this leads to prohibitive sample complexity and unstable convergence, limiting its applicability.

\paragraph{Continuous-time RL.} Continuous-time RL provides a rigorous theoretical framework for decision-making, established through entropy-regularized HJB equations~\cite{wang2020reinforcement} and martingale-based policy evaluation~\cite{jia2022policy, jia2023q}. Recent works have extended this foundation to model-free temporal difference methods~\cite{settai2025temporal} and diffusion-guided policies~\cite{hua2025continuous}. Unlike existing methods that rely on general policy iteration or direct HJB solving, our approach leverages the specific recursive structure of the path integral value function. This allows us to bypass high-order HJB terms and employ TD learning for stable, off-policy value estimation.

\section{Recursive Path Integral Theory}
\label{sec:method}

\subsection{Path Integral: A Recursive Form}
\label{subsec:path_integral}

The core insight of our work is that the path integral representation of optimal value function (\cref{eq:path_integral_form_of_v}), has the following recursive form.
\begin{proposition}
\label{prop:value_recursion}
For any $0 < t < s$, the following expectation equation holds:
\begin{equation}
\label{eq:key_recursive_value}
\begin{aligned}
\exp&\left( - V(X_t, t) \right)
=   \mathbb{E}_{\mathbb{P}^0}\Bigg[ \exp\left( - V(X_s, s) \right) 
 \exp\left(- \int_{t}^{s} f(X_r,r)\,\mathrm{d}r \right) \, \Bigg | \, \mathcal{F}_t \Bigg], \\
 &\quad V(x, 1) = g(x),
\end{aligned}
\end{equation}
where $X_t \in \mathbb{R}^d$ is the state of uncontrolled stochastic process.
\end{proposition}
\begin{proof}
The result follows directly from the tower property of conditional expectation. By decomposing the path integral in \cref{eq:path_integral_form_of_v} into the intervals $[t, s]$ and $[s, 1]$, and taking the conditional expectation over the trajectory from $s$ to the terminal time, we recover the term $\exp\left( - V(X_s, s) \right)$ by definition. The detailed proof can be found in Appendix \ref{appendix:unbiasedness_properties}.
\end{proof}
Crucially, the recursive structure in \cref{eq:key_recursive_value} suggests that dynamic programming techniques can be adapted to solve for the value function iteratively.

\subsection{Properties of Recursive Path Integral}
\label{subsec:properties_path_integral}
This section investigates whether the above formulation is practically applicable, primarily by analyzing its convergence behavior and variance properties. By treating the recursion in \cref{eq:key_recursive_value} as an update operator, we firstly formulate the following bootstrapping rule.

\begin{definition}
\textit{For every $t \in [0, 1)$ and a fixed $\varepsilon > 0$, the updating rule at the $k$-th iteration is:}
\begin{equation}
\label{eq:value_iteration}
\begin{aligned}
V_{k}(X_t, t) = - \log \mathbb{E}_{\mathbb{P}^0}\Bigg[ \exp\left( - \tilde{V}_{k - 1}(X_{t + \varepsilon}, t + \varepsilon) \right) \exp\left(-\int_{t}^{t + \varepsilon} f(X_r,r)\,\mathrm{d}r \right) \, \Bigg | \, \mathcal{F}_t \Bigg],
\end{aligned}
\end{equation}
where $\tilde{V}_{k-1}$ is defined such that $\tilde{V}_{k-1}(x, s) = V_{k-1}(x, s)$ for $s < 1$ and $\tilde{V}_{k-1}(x, 1) = g(x)$.
\end{definition}
 To ensure the theoretical validity of this iterative scheme, we must guarantee that it converges to the true solution. The following theorem establishes this convergence property under mild assumptions.
\begin{theorem}
\label{thrm:convergence_of_bootstrapping}
In addition to the notation established in \cref{label:Preliminaries}, we further assume that:
\begin{enumerate}
\item[A1.] The running cost is bounded from below by a positive constant, i.e., $f(x, t) \geq f_{\inf} > 0$;
\item[A2.] The optimal value function is bounded from below, i.e., $V(x, t) \geq V_{\inf}$;
\item[A3.] The underlying SDE satisfies standard Lipschitz and growth conditions, ensuring the transition semigroup possesses the Feller property.
\end{enumerate}
If the sequence $\{V_k\}_{k=1}^\infty$ satisfies the recursion \eqref{eq:value_iteration} for all $t$, then it converges to the optimal value function:
\begin{equation}
\lim_{k \to \infty} V_k(x, t) = V(x, t).
\end{equation}
\end{theorem}
\begin{proof}
    Available in Appendix \ref{appendix:convergence_analysis_of_recursive_update}.
\end{proof}
While Theorem \ref{thrm:convergence_of_bootstrapping} guarantees correctness, practical feasibility depends on sample efficiency. A fundamental advantage of our recursive value iteration over direct Monte Carlo estimation (which samples full trajectories to the terminal time) is the significant reduction in estimation variance. We formalize this benefit through the following variance decomposition analysis.

\begin{proposition}
\label{prop:variance_decomposition}
For any $t \in [0, 1)$ and $s \in (t,1]$,
\begin{equation}
\label{eq:variance_decomposition}
\begin{aligned}
& \mathbb{V}(Z_t|\mathcal{F}_t) = \mathbb{V}\Big(\mathbb{E}[Z_s|\mathcal{F}_s] \exp\Big(- \int_{t}^{s} f(X_r,r)\,\mathrm{d}r \Big) \Big| \mathcal{F}_t\Big) \\
& + \mathbb{E} \Big[\mathbb{V}(Z_s|\mathcal{F}_s) \exp\Big(- 2\int_{t}^{s} f(X_r,r)\,\mathrm{d}r \Big)|\mathcal{F}_t \Big],
\end{aligned}
\end{equation}
where the expectation and variance are all taken over $\boldsymbol{X} \sim \mathbb{P}^{0}$ and the notation $Z_t$ is defined as
\begin{equation}
    Z_t = \exp\left(-\mathcal{W}(X, t)\right).
\end{equation}
\end{proposition}
%
\begin{proof}
    Available in Appendix
    \ref{appendix:variance_analysis}.
\end{proof}
Equation \eqref{eq:variance_decomposition} offers a clear interpretation of the efficiency gain. The term on the left-hand side (LHS) represents the variance of the naive Monte Carlo estimator. The first term on the right-hand side (RHS) corresponds to the variance of our recursive estimator (assuming $V_k$ has converged to $V$). Consequently, the second term on the RHS quantifies the explicit \textit{variance reduction} achieved by our method. This reduction is particularly substantial when the current time $t$ is far from the terminal time, highlighting the superiority of the \cref{eq:key_recursive_value} in long-horizon problems.

\section{Path Integral Value Matching (PI-VM)}
\label{sec:pi_td_algorithm}

Guided by the recursive theoretical foundation, we propose the PI-VM algorithm. In this section, we will detail the derivation of the recursive TD loss, its extension to off-policy scenarios using the Girsanov theorem, and the practical implementation specifics.

\subsection{TD Learning}
\label{subsec:pi_td_learning}
To implement the update rule \cref{eq:key_recursive_value} in a practical algorithm, we parameterize the value function $V_k(x, t)$ using a neural network $V_{\theta}(x, t)$ with parameters $\theta$.
Consequently, the value iteration is formulated as a deep learning optimization task.
Specifically, the loss function of our PI-VM learning is defined as follows:

\begin{definition}[\textbf{Path Integral TD Loss}]
\label{def:pi_td_loss}
Consider a state $x$ with corresponding time $t$, a fixed lookahead horizon determined by $s = \min(1, t + M\Delta t)$, the PI-VM loss is defined as:
\begin{equation}
\label{eq:loss_base}
    \ell(\theta) =  \left\| V_{\theta}(x,t) - \hat{V}_{\theta}(x, t, s) \right\|^2,
\end{equation}
where the target value $\hat{V}_{\theta}$ is estimated via a Monte Carlo approximation of the path integral. Specifically, given $N$ sample trajectories, the target is computed as:
\begin{equation}
\label{eq:target_value}
\begin{aligned}
    \hat{V}_{\theta}(x, t, s) = -\log \frac{1}{N} \sum_{j=1}^N \exp\left( - \hat{\mathcal{W}}(X^{(j)}_{[t,s]}) - \hat{\mathcal{G}}(X^{(j)}_{s}) \right) ,
\end{aligned}
\end{equation}
where each trajectory $X^{(j)}_{[t,s]}$ follows the uncontrolled path measure $\mathbb{P}^0_{[t,s]}$ initialized at $X_t^{(j)} = x$. The functional terms are defined as:
\begin{equation}
\label{eq:functional_1}
\begin{aligned}
    \hat{\mathcal{W}}(X^{(j)}_{[t,s]}) = \int_{t}^{s} f(X^{(j)}_r, r) \, \mathrm{d}r, \quad
    \hat{\mathcal{G}}(X^{(j)}_{s}) = 
    \begin{cases}
        V_{\theta}(X^{(j)}_s, s) & s < 1, \\
        g(X^{(j)}_s)           & s = 1.
    \end{cases}
\end{aligned}
\end{equation}
\end{definition}

\textbf{Remark.} While our method also relies on TD learning to solve the dynamic programming problem, similar to reinforcement learning (RL) approaches, its underlying nature is fundamentally different. We highlight the distinctions with classical RL and risk-sensitive RL~\cite{mihatsch2002risk, noorani2023exponential} below.

\textit{PI-VM vs. classical RL.} A common misconception is that our method corresponds to a continuous version of the classical Bellman or optimal Bellman equations~\cite{bellman1966dynamic}. In fact, directly taking the discrete-to-continuous limit of these equations does not yield our log-sum-exp formulation. Instead, our formulation arises directly from path integral control.

\textit{PI-VM vs. risk-sensitive RL.} Although the Bellman equation in risk-sensitive RL~\cite{mihatsch2002risk, noorani2023exponential} also takes a log-sum-exp form, its object differs: risk-sensitive RL applies it to the value function, whereas in our method, the log-sum-exp formulation applies to the optimal value function.

\begin{algorithm*}[t] 
   \caption{Path Integral Value Matching (PI-VM)}
   \label{alg:pi_td}
   \begin{algorithmic}[1]
      \State \textbf{Input:} Problem Setup, Hyperparameters (Batch size $B$, Sample Size $N$, Sample Steps $M$, EMA $\tau$, Learning Rate $\eta$).
      \State \textbf{Initialize:} Replay Buffer $\mathcal{D}$, Value Network $V_\theta$, Target Network $V_{\hat{\theta}}$.
      
      \For{each training iteration}
         \If{time to update buffer}
            \State Rollout trajectories using current best control.
            \For{each sampled state in trajectories}
               \State Simulate $N$ short-term $M$-step path branches.
               \State Compute functionals $\hat{\mathcal{W}}, \hat{\mathcal{S}}$ for each branch following \cref{eq:functional_1,eq:functional_2}.
               \State Store transition tuple into $\mathcal{D}$.
            \EndFor
         \EndIf

         \State Sample a batch $\mathcal{B}$ of transitions from $\mathcal{D}$.
         \State Compute target functional $\hat{\mathcal{G}}$ and value $\hat{V}_{\hat{\theta}}$ following \cref{eq:functional_1,eq:target_value_controlled}.
         
         \State Compute TD loss: $\mathcal{L}(\theta) = \frac{1}{|\mathcal{B}|} \sum (V_\theta(x, t) - \hat{V}_{\hat{\theta}})^2$.
         \State Update value network: $\theta \leftarrow \theta - \eta \nabla_\theta \mathcal{L}(\theta)$ and  target value network: $\hat{\theta} \leftarrow \tau \theta + (1-\tau)\hat{\theta}$.
      \EndFor
      \State \textbf{Output:} Optimal Value Function $V_\theta$.
   \end{algorithmic}
\end{algorithm*}

\subsection{Off-Policy Improvement}
\label{subsec:off_policy_improvement}

While \cref{eq:loss_base} effectively mitigates variance in many scenarios, it faces challenges when the distribution of the base path diverges significantly from the target optimal path. To avoid the high variance associated with over-exploration and the mode collapse caused by over-exploitation, a dynamic sampling strategy is essential. Consequently, an off-policy improvement becomes imperative. We begin by formulating the off-policy path integral TD loss via the Girsanov theorem, followed by a proof showing that the variance of our estimator asymptotically vanishes as the control policy $u$ approaches optimality.

%


According to the celebrated  Girsanov's Theorem~\cite{oksendal2013stochastic}, the conditional density ratio of $\mathbb{P}^{0}$ with respect to $\mathbb{P}^u$ over the interval $[t, s]$, conditioned on $\mathcal{F}_t$, is given by
\begin{equation}
\label{eq:functional_2}
\begin{aligned}
&\frac{\mathrm{d} \mathbb{P}_{[t,s]}^{0}}{\mathrm{d} \mathbb{P}_{[t,s]}^u} \Bigg|_{\mathcal{F}_t} (X^u_{[t,s]})= \exp (- \hat{\mathcal{S}}(X^u_{[t,s]},u) ) = \exp \Bigg(- \int_t^s u(X^u_r, r) \, \mathrm{d} B_{r} - \int_{t}^{s} \frac{1}{2}\left\|u(X^u_r, r)\right\|^{2} \,\mathrm{d}r \Bigg).
\end{aligned}
\end{equation}

Therefore, we obtain the off-policy version of PI-VM loss:
\begin{definition}[\textbf{Off-Policy PI-VM Loss}]
\label{def:pi_td_loss_u}
The off-policy PI-VM loss has the same setup with \cref{def:pi_td_loss},
\begin{equation}
\label{eq:loss_off_policy}
    \ell^u(\theta) =  \left\| V_{\theta}(x,t) - \hat{V}_{\theta}(x, u, t, s) \right\|^2,
\end{equation}
where the only difference is that the uncontrolled target value is replaced with the controlled one:
\begin{equation}
\label{eq:target_value_controlled}
\begin{aligned}
    \hat{V}_{\theta}(x, u, t, s) = -\log \frac{1}{N} \sum_{j=1}^N  \exp\left( - \hat{\mathcal{W}}(X^{(j)}_{[t,s]})  - \hat{\mathcal{S}}(X^{(j)}_{[t,s]},u) - \hat{\mathcal{G}}(X^{(j)}_s)\right) ,
\end{aligned}
\end{equation}
where each trajectory $X^{(j)}_{[t,s]}$ follows the controlled path measure $\mathbb{P}^u_{[t,s]}$ initialized at $X_t^{(j)} = x$.  
\end{definition}
This off-policy version of loss has substantially small variance when the sampling policy $u$ is close to the optimal policy $u^*$.
To see this, we first define
\begin{equation}
\label{eq:M_definition}
\begin{aligned}
\mathcal{M}(X_{[t,s]}) &= \exp \Bigg( - V(X_s,s)  -  \int_t^s u\left(X_{r}, r\right) \mathrm{d} B_{r}  - \int_{t}^{s} \left[ f(X_r,r) + \frac{1}{2}\left\|u\left(X_{r}, r\right)\right\|^{2} \right] \,\mathrm{d}r \Bigg).
\end{aligned}
\end{equation}
Here, we utilize the true optimal value function in the analysis.
Since $V_\theta(x, t)$ approaches to $V(x, t)$, the asymptotic behavior of $\hat{V}_{\theta}(x, u, t, s)$ can be characterized in terms of $\mathcal{M}(X_{[t,s]})$.
%
%
The variance of $\mathcal{M}(X_{[t,s]})$ is characterized by the following theorem.
\begin{theorem}
\label{thrm:var_bound}
Suppose that the control $u$ satisfies $\sup_{x,r} \| u^*(x, r) - u(x, r) \|^2 \leq \kappa$. Then, the conditional variance of $\mathcal{M}(X_{[t,s]})$ satisfies the following upper bound 
\begin{equation}
\begin{aligned}
\mathbb{V}_{p^u}(\mathcal{M}(X_{[t,s]})|\mathcal{F}_t) 
\leq  \exp\left(-2V(X_t,t)\right) \left[ \exp\left( \kappa(s-t) \right) - 1 \right].
\end{aligned}
\end{equation}
\end{theorem}
\begin{proof}
    Available in Appendix \ref{appendix:variance_analysis}.
\end{proof}
The theorem illustrates that the variance bound diminishes as a function of the control error $\kappa$. 
As the current control $u$ converges toward $u^*$, the upper bound becomes increasingly tighter, ensuring that the estimation variance asymptotically vanishes.
\begin{table}
  \caption{Quantitative evaluation  on three unimodal SOC tasks.}
  \label{tab:single_model}
  \centering
  \begin{small}
    \begin{sc}
    \resizebox{\columnwidth}{!}{%
      \begin{tabular}{l cc cc cc}
        \toprule
        & \multicolumn{2}{c}{OU Linear} & \multicolumn{2}{c}{OU Quadratic (easy)} & \multicolumn{2}{c}{OU Quadratic (hard)} \\
        
        \cmidrule(lr){2-3} \cmidrule(lr){4-5} \cmidrule(lr){6-7}
        
        Method & Control $\mathcal{L}_2$ ($\downarrow$) & Time ($\downarrow$) & Control $\mathcal{L}_2$ ($\downarrow$) & Time ($\downarrow$) & Control $\mathcal{L}_2$ ($\downarrow$) & Time ($\downarrow$) \\
        \midrule
        
        RE    & 0.00003 $\pm$ 0.00000 & 87$ms$ & 0.00019 $\pm$ 0.00001 & 54$ms$ & 0.00453 $\pm$ 0.00034 & 50$ms$ \\
        CE    & 0.00056 $\pm$ 0.00035 & 41$ms$ & 0.00215 $\pm$ 0.00019 & 22$ms$ & 1.46971 $\pm$ 0.01211 & 21$ms$ \\
        VAR   & 0.00032 $\pm$ 0.00016 & 38$ms$ & 0.09234 $\pm$ 0.00057 & 22$ms$ & 1.46971 $\pm$ 0.01211 & 21$ms$ \\
        LVAR  & 0.00017 $\pm$ 0.00005 & 41$ms$ & 0.00061 $\pm$ 0.00003 & 22$ms$ & 0.01521 $\pm$ 0.00091 & 21$ms$ \\
        REINFORCE  & 0.00097 $\pm$ 0.00033 & 50$ms$ & 0.01615 $\pm$ 0.00035 & 28$ms$ & 0.20843 $\pm$ 0.01778 & 27$ms$ \\
        SOCM  & \textbf{0.00000} $\pm$ \textbf{0.00000} & 128$ms$ & 0.00044 $\pm$ 0.00001 & 183$ms$ & 1.46971 $\pm$ 0.01211 & 194$ms$ \\
        SOCM-A    & \textbf{0.00000} $\pm$ \textbf{0.00000} & 41$ms$ & 0.00073 $\pm$ 0.00005 & 21$ms$ & 1.46971 $\pm$ 0.01211 & 21$ms$ \\
        AM    & \textbf{0.00000} $\pm$ \textbf{0.00000} & 41$ms$ & 0.00028 $\pm$ 0.00001 & 21$ms$ & 0.00674 $\pm$ 0.00036 & 21$ms$ \\
        PI-VM (Ours) & \textbf{0.00000} $\pm$ \textbf{0.00000} & \textbf{8$ms$} & \textbf{0.00010} $\pm$ \textbf{0.00000} & \textbf{7$ms$} & \textbf{0.00078} $\pm$ \textbf{0.00005} & \textbf{7$ms$} \\
        \bottomrule
      \end{tabular}
    }
    \end{sc}
  \end{small}
\end{table}
\subsection{Algorithm Design}
\label{subsec:algorithm_design}
As summarized in \cref{alg:pi_td}, training alternates between generating trajectories to populate a Replay Buffer and optimizing the value function via TD learning. Trajectories are rolled out using the current best control, with short $M$-step path branches sampled for each state and stored in the buffer. The TD target is estimated from these branches using $N$ Monte Carlo samples, balancing variance reduction and update depth. During optimization, mini-batches are sampled uniformly from the buffer, and the value network is updated to minimize the squared difference between predicted and target values, while a target network updated via exponential moving average stabilizes training. 
\begin{table*}[t]
  \caption{Scalability analysis on the Quadratic OU Easy task with increasing dimensions.}
  \label{tab:scalability}
  \centering
  \begin{small}
    \begin{sc}
    \resizebox{\columnwidth}{!}{%
      \begin{tabular}{c cc cc cc}
        \toprule

        & \multicolumn{2}{c}{SOCM} & \multicolumn{2}{c}{AM} & \multicolumn{2}{c}{PI-VM (Ours)} \\

        \cmidrule(lr){2-3} \cmidrule(lr){4-5} \cmidrule(lr){6-7}

        Dim & Control $\mathcal{L}_2$  ($\downarrow$) & Time ($\downarrow$) & Control $\mathcal{L}_2$  ($\downarrow$) & Time ($\downarrow$) & Control $\mathcal{L}_2$  ($\downarrow$) & Time ($\downarrow$) \\
        \midrule
        
        10  & {0.00027} $\pm$ {0.00001} & 130$ms$ & {0.00026} $\pm$ {0.00001} & 29$ms$ & \textbf{0.00006} $\pm$ \textbf{0.00000} & \textbf{7$ms$} \\

        20  & {0.00044} $\pm$ {0.00001} & 182$ms$ & {0.00027} $\pm$ {0.00002} & 29$ms$ & \textbf{0.00011} $\pm$ \textbf{0.00000} & \textbf{8$ms$} \\


        80  & \multicolumn{2}{c}{Out of Memory} & 0.02960 $\pm$ 0.00024 & 28$ms$ & \textbf{0.00044} $\pm$ \textbf{0.00004} & \textbf{12$ms$} \\
        
        
        150 & \multicolumn{2}{c}{Out of Memory} & 0.07620 $\pm$ 0.00047 & 32$ms$ & \textbf{0.00080} $\pm$ \textbf{0.00017} & \textbf{18$ms$} \\
        
        200 & \multicolumn{2}{c}{Out of Memory} & 0.08983 $\pm$ 0.00051 & 38$ms$ & \textbf{0.00256} $\pm$ \textbf{0.00033} & \textbf{21$ms$} \\
        \bottomrule
      \end{tabular}
    }
    \end{sc}
  \end{small}
  
\end{table*}

\begin{figure}[ht]
  \centering
  \begin{subfigure}{0.24\textwidth}
    \includegraphics[width=\linewidth]{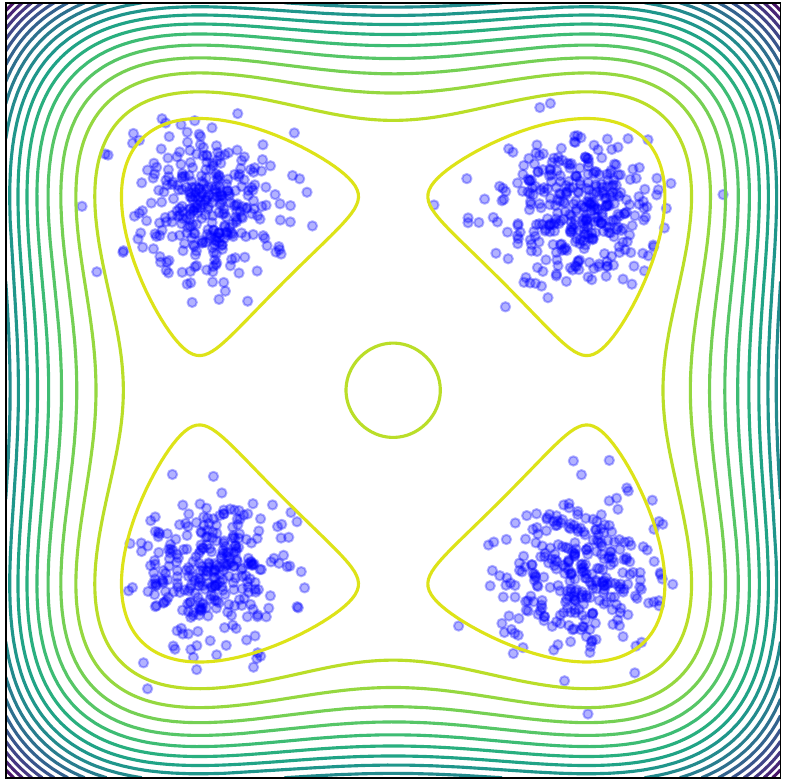}
  \end{subfigure}\hfill
  \begin{subfigure}{0.24\textwidth}
    \includegraphics[width=\linewidth]{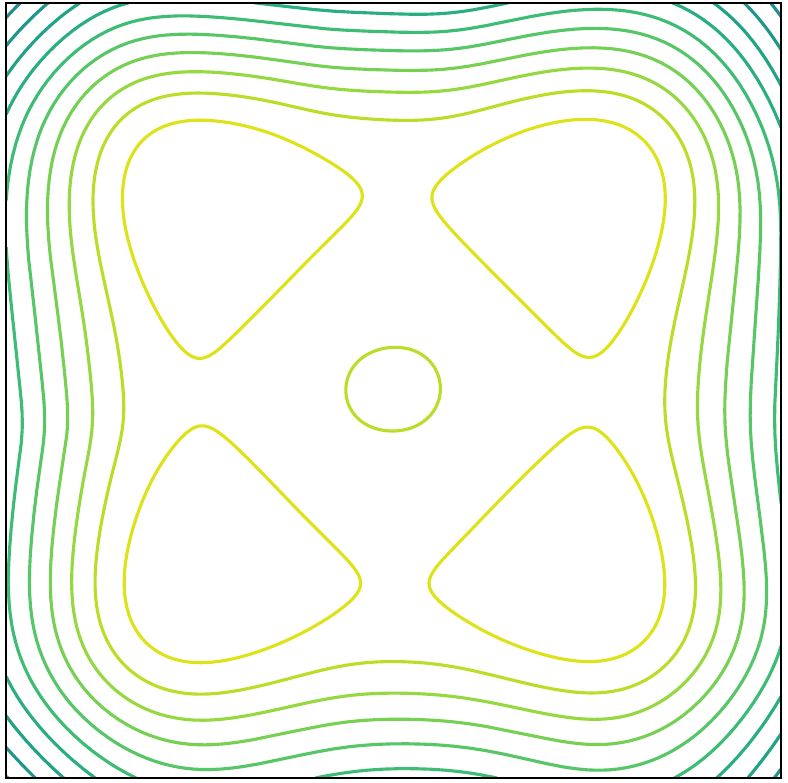}
  \end{subfigure}\hfill
  \begin{subfigure}{0.24\textwidth}
    \includegraphics[width=\linewidth]{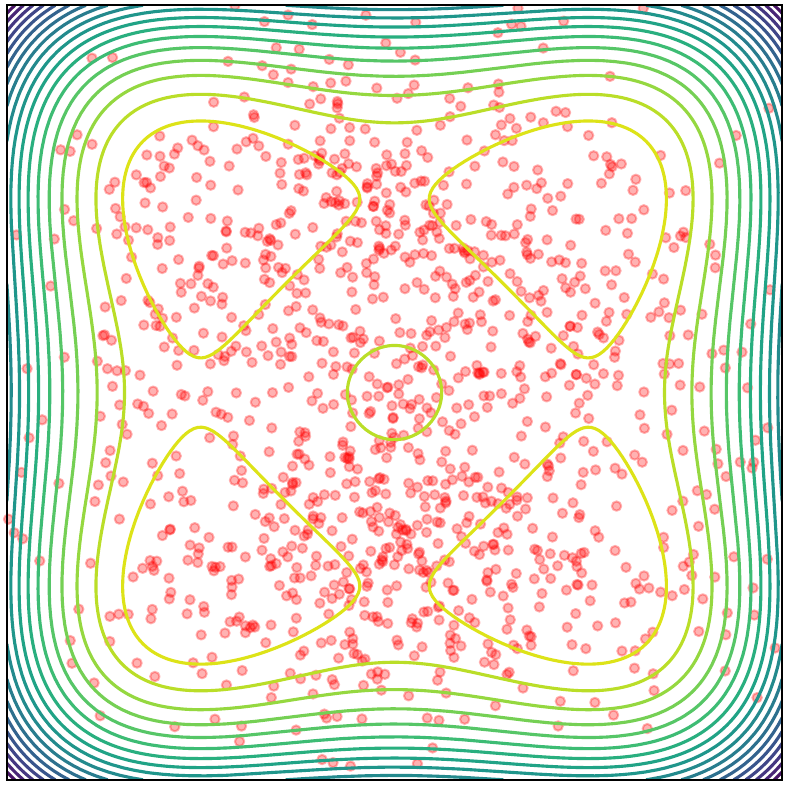}
  \end{subfigure}\hfill
  \begin{subfigure}{0.24\textwidth}
    \includegraphics[width=\linewidth]{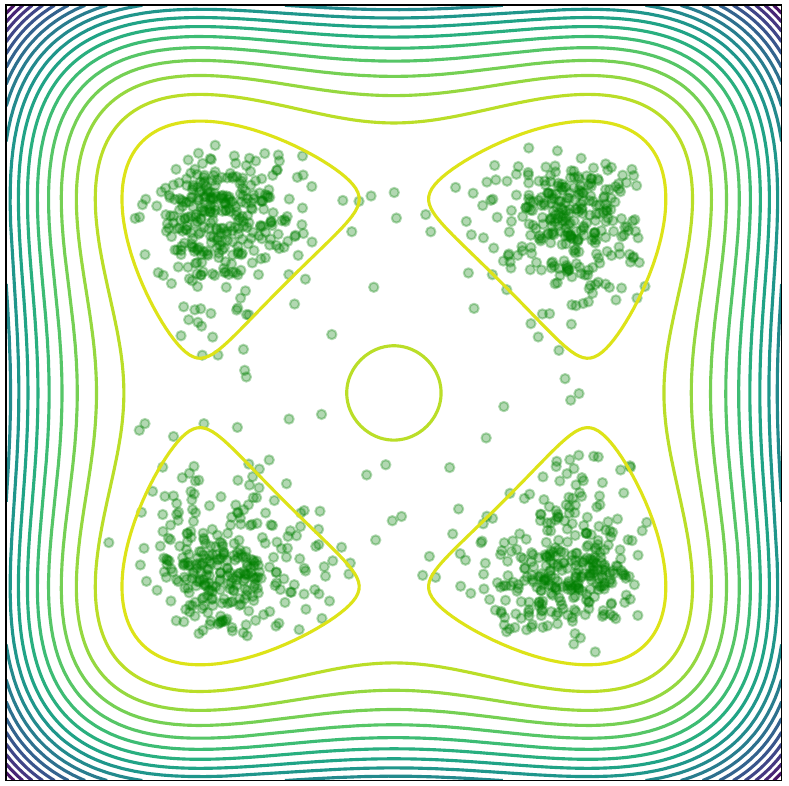}
  \end{subfigure}
  
  \caption{
    Sampling and value comparison on Many Well sampling task. 
    (Left) Ground Truth samples and Ground Truth value landscape; 
    (Middle Left) Our learned value landscape; 
    (Middle Right) Samples from AM; 
    (Right) Samples from our method.
  }
  \label{fig:many_well}
\end{figure}

\section{Experiments}
We will begin by providing an overview of the problems and baselines, then describe the training setups, and finally introduce evaluation metrics employed across all problems. We provide a detailed introduction in Appendix \ref{appendix:soc_exp_settings} and \ref{appendix:sampling_exp_settings}.

\textbf{Problems and Baselines.}
We consider five tasks that we adapt from ~\cite{domingo2024stochastic, blessing2025trust}: \textit{Linear Ornstein Uhlenbeck}, \textit{Quadratic Ornstein
Uhlenbeck (easy)}, \textit{Quadratic Ornstein Uhlenbeck (hard)}, \textit{Gaussian Mixture Models}, and \textit{Many Well}. The first three settings are unimodal SOC control problems and the last two are multimodal sampling problems. We compare 7 policy-based baselines: \textit{relative entropy (RE)}, \textit{cross entropy (CE)}~\cite{holdijk2023stochastic}, \textit{variance (VAR)}~\cite{nusken2021solving}, \textit{log-variance (LVAR)}~\cite{nusken2021solving}, \textit{adjoint matching (AM)}~\cite{domingo-enrich2025adjoint}, \textit{stochastic optimal control matching (SOCM)}~\cite{domingo2024stochastic}, and \textit{stochastic optimal control matching with adjoint loss (SOCM-A)}~\cite{domingo2024stochastic}. 

\textbf{Training and Evaluation.} For the first three experiments, we strictly align our training hyperparameters with those in \textit{SOCM}~\cite{domingo2024stochastic}. For the last two problems, we follow the configuration described in \cite{blessing2025trust}. Across all experiments, we utilize a 6-layer fully-connected neural network with GELU activations~\cite{hendrycks2016gaussian} and Adam optimizer~\cite{kingma2014adam}. For the first four benchmarks, where the ground truth optimal control is analytically tractable, we adopt the control $\mathcal{L}_2$ Error as our primary evaluation metric: $\mathcal{L}_2 (u,u^*) := \frac{1}{d}\mathbb{E}_{\mathbb{P}^{u^*}} \left[ \int_0^1 \|u^* - u\|^2 (X^{u^*}_s, s) \, \mathrm{d}s \right]$. This metric quantifies the Euclidean distance between the learned policy and the ground truth optimal control. It is worth noting that this comparison is particularly stringent for our method, as we infer the control indirectly via the gradient of the value function (i.e., $u = - \sigma^T\nabla V$), whereas policy-based baselines optimize the policy directly. All experiments were conducted on a single NVIDIA RTX 4060 GPU, and reported results are averaged over four independent runs with distinct random seeds.

\begin{table*}[t] 
    \centering
    \label{tab:combined_tables}
    
    \begin{minipage}{0.38\textwidth} 
          \caption{Ablation study on {sample size} and {forward steps}. In each cell, the {top value} is the control $\mathcal{L}_2$ error, and the {bottom value} is the Runtime.}
          \label{tab:ablation_stacked}
          \centering
          
          \newcommand{\mycell}[2]{
            \begin{tabular}{@{}c@{}}
              #1 \\  
               #2 
            \end{tabular}
          }
  
        \begin{small}
        \begin{sc}
        \resizebox{\columnwidth}{!}{%
          \renewcommand{\arraystretch}{1.3} 
          
          \begin{tabular}{c ccc}
            \toprule
            & \multicolumn{3}{c}{Forward Steps $M$} \\
            \cmidrule(lr){2-4}
            Samples $N$ & 4 & 8 & 16 \\
            \midrule
        
            4 & \mycell{0.00013}{7$ms$} 
              & \mycell{0.00011}{7$ms$} 
              & \mycell{0.00010}{8$ms$} \\
        
            8 & \mycell{0.00012}{7$ms$} 
              & \mycell{\textbf{0.00010}}{\textbf{7$ms$}}
              & \mycell{0.00010}{9$ms$} \\
        
            16 & \mycell{0.00011}{9$ms$} 
               & \mycell{0.00010}{13$ms$} 
               & \mycell{0.00009}{18$ms$} \\
        
            \bottomrule
          \end{tabular}
        }
        \end{sc}
        \end{small}
    \end{minipage}
    \hfill 
    \begin{minipage}{0.58\textwidth}
          \caption{Quantitative evaluation on the multimodal Gaussian Mixture Model (GMM) sampling task. "Dist." refers to the distance between component means, and "Var." refers to the variance of each component.}
  \label{tab:gmm_results}
  \centering
  \begin{small}
    \begin{sc}
    \resizebox{\linewidth}{!}{
      \begin{tabular}{c @{\hspace{1mm}} c @{\hspace{4mm}} cc}
        \toprule
        \multicolumn{2}{c}{Setting} & \multicolumn{2}{c}{Control $\mathcal{L}_2$ ($\downarrow$)} \\
        \cmidrule(r){1-2} \cmidrule(l){3-4}
        Dist. & Var. & AM & PI-VM (Ours) \\
        \midrule
        \multirow{2}{*}{Close} 
          & Large & 0.016 $\pm$ 0.000 & \textbf{0.010} $\pm$ \textbf{0.000} \\
          & Small & 6.996 $\pm$ 0.025 & \textbf{0.025} $\pm$ \textbf{0.003} \\
        \cmidrule{1-4}
        \multirow{2}{*}{Far} 
          & Large & 0.018 $\pm$ 0.002 & \textbf{0.011} $\pm$ \textbf{0.000} \\
          & Small & 4.317 $\pm$ 0.260 & \textbf{0.024} $\pm$ \textbf{0.006} \\
        \bottomrule
      \end{tabular}
    }
    \end{sc}
  \end{small}
    \end{minipage}
\end{table*}

\subsection{Quantitative Evaluation}
\textbf{Unimodal SOC Tasks.}
As shown in~\cref{tab:single_model}, we compare the mean and standard deviation of control $\mathcal{L}_2$ error against other baseline methods across three unimodal SOC control tasks. While all methods perform comparably on the simple \textit{Linear} task, our \textbf{PI-VM} method significantly outperforms baselines on the \textit{Quadratic} tasks. Notably, in the \textit{``Hard''} setting, state-of-the-art methods like \textit{SOCM} and \textit{SOCM-A} fail to converge. In contrast, our value-based approach successfully approximates the global landscape, achieving the lowest error. Furthermore, \textbf{PI-VM} demonstrates superior efficiency, running $10$--$20\times$ faster than baselines. This directly addresses the critical bottleneck of prior works: by avoiding computationally prohibitive forward trajectory simulations and stochastic back-propagation, our method achieves both higher precision and real-time inference speeds.

\textbf{Multimodal Sampling Task.}  We further evaluate the proposed method on the 20 dimensions multimodal \textit{Gaussian Mixture Model (GMM)} task under varying difficulty settings, defined by the distance between modes (Dist.) and the variance of modes (Var.). As shown in \cref{tab:gmm_results}, our method consistently outperforms the baseline across all configurations. Notably, our approach demonstrates superior robustness in scenarios with \textit{Small Variance} (i.e., highly peaked distributions). In these challenging settings, the baseline method suffers from catastrophic failure. This is likely because policy-based methods struggle with the vanishing gradients typical of sharp energy landscapes. In contrast, our value-based approach maintains exceptional stability, achieving low errors. This result highlights the efficacy of our method in handling complex, non-convex landscapes where traditional baselines diverge.

\textbf{High-dim Scalability.} \cref{tab:scalability} investigates the scalability on the \textit{Quadratic OU Easy} task with state dimensions increasing from 10 to 200. The results reveal severe limitations in baseline methods: \textit{SOCM} quickly encounters memory bottlenecks (OOM) at $d=80$ due to the prohibitive cost of storing full trajectories, while \textit{AM} suffers from optimization instability, resulting in error explosion and eventual divergence at $d=200$. In stark contrast, ours \textbf{PI-VM} effectively breaks the curse of dimensionality, maintaining robust convergence and real-time inference speeds even at $d=200$. This confirms that by bypassing the computational burden of forward simulation and adjoint back-propagation, our value-based formulation offers a uniquely scalable solution for high-dimensional stochastic control.
\subsection{Mode Coverage Qualitative Analysis}
To evaluate the algorithm's capability in handling multimodal distributions and traversing high-energy barriers, we conducted a visualization experiment on a 50 dimensions \textit{Many-well} potential landscape. As shown in the top row of \cref{fig:many_well}, our learned Value Function (Top-Right) accurately reconstructs the underlying non-convex geometry, providing a correct guidance for control. The bottom row highlights the sampling performance: the \textit{Adjoint Matching} baseline (Bottom-Left, red) suffers from high variance, with a significant number of samples failing to converge and scattering across high-energy barriers. In contrast, our method (Bottom-Right, green) generates high-fidelity samples that are strictly confined to the stable equilibria. This visual evidence demonstrates that our approach effectively mitigates the mode-mixing issues and achieves a distribution density that aligns closely with the Ground Truth.
\subsection{Ablation Studies}
\label{subsec:ablation_studies}
We investigated the impact of varying the number of Monte Carlo samples ($N$) and the forward TD simulation steps ($M$) on control performance and computational efficiency. As summarized in \cref{tab:ablation_stacked}, increasing both $N$ and $M$ generally leads to lower control $\mathcal{L}_2$ error, albeit at the cost of increased runtime. We observe diminishing returns in performance improvement beyond certain thresholds. For instance, extending the steps from $M=8$ to $M=16$ (at $N=8$) yields no reduction in loss but increases runtime. Similarly, doubling the samples to $N=16$ significantly increases computational overhead without improving accuracy. Consequently, the empirical results suggest that the configuration ($N=8, M=8$) offers the optimal trade-off between accuracy and inference speed, serving as our default setting. Additional ablation studies on network depth, batch size, and sample size are provided in~\cref{tab:different_batch,tab:fixed_batch,tab:network_depth}, demonstrating the robustness and efficiency of our method across different training configurations.

\section{Conclusion}

We presented PI-VM, a novel value-based algorithm for high-dimensional Stochastic Optimal Control. By revisiting Path Integral Control and deriving a temporal recursive formulation, we effectively eliminated the need for computationally expensive and high variance full-trajectory simulations. Our method leverages TD learning, experience replay buffer, and Girsanov theorem to achieve stable and efficient off-policy training. Experimental results confirm that PI-VM significantly outperforms existing policy-based baselines in terms of training speed, numerical stability, and scalability to high-dimensional state spaces. Despite these advantages, as a value-based approach, PI-VM still requires computationally expensive automatic differentiation to obtain the control signal, which can limit runtime efficiency in practice.

{
    \small
    \bibliographystyle{unsrt}
    \bibliography{references}
}


\newpage
\appendix






\section{Recursive Structure of the Optimal Value Function}
\label{appendix:unbiasedness_properties}
\addcontentsline{apctoc}{section}{A. Recursive Structure of the Optimal Value Function}
In this section, we demonstrate detailed proofs for the conclusion about the recursive structure of the optimal value function in Proposition \ref{prop:value_recursion}.
\textbf{Restatement of Proposition \ref{prop:value_recursion}.} 
\textit{For any $0 < t < s \leq 1$, the optimal value function satisfies the following unbiased recursive relations under the base measure $\mathbb{P}^{0}$ and the controlled measure $\mathbb{P}^u$:}
\begin{equation*}
\begin{aligned}
\exp&\left( - V(X_t, t) \right)  \\ &= \mathbb{E}_{\mathbb{P}^{0}} \left[ \exp\left( - V(X_s, s) - \int_t^s f(X_r, r) \mathrm{d}r \right) \Bigg| \mathcal{F}_t \right] \\
&= \mathbb{E}_{\mathbb{P}^{u}} \left[ \exp\left( - V(X_s, s) - \int_t^s \left[ f(X_r, r) + \frac{1}{2}\left\|u\left(X_{r}, r\right)\right\|^{2} \right] \mathrm{d}r - \int_t^s u\left(X_{r}, r\right) \mathrm{d} B_{r} \right) \Bigg| \mathcal{F}_t \right].
\end{aligned}
\end{equation*}

\begin{proof}
Following the framework established in Section \ref{label:Preliminaries}, the value function admits the following Feynman-Kac representation:
\begin{equation}
\label{eq:FK_rep}
\exp \left( - V(X_t,t) \right) = \mathbb{E}_{\mathbb{P}^{0}}\left[\exp\left( - g(X_1) - \int_{t}^{1} f(X_r,r)\,\mathrm{d}r\right)\,\Big|\, \mathcal{F}_t\right].
\end{equation}
By invoking the tower property of conditional expectation, we can decompose the integral over the interval $[t, 1]$ into $[t, s]$ and $[s, 1]$ for any $s \in (t, 1)$:
\begin{equation*}
\begin{aligned}
\exp & \left( - V(X_t,t) \right) \\ &= \mathbb{E}_{\mathbb{P}^{0}}\left[\mathbb{E}_{\mathbb{P}^{0}}\left[\exp\left( - g(X_1) - \int_{s}^{1} f(X_r,r)\,\mathrm{d}r - \int_{t}^{s} f(X_r,r)\,\mathrm{d}r \right) \, \Big| \, \mathcal{F}_s \right] \,\Big|\, \mathcal{F}_t\right] \\
&= \mathbb{E}_{\mathbb{P}^{0}}\left[\mathbb{E}_{\mathbb{P}^{0}}\left[\exp\left( - g(X_1) - \int_{s}^{1} f(X_r,r)\,\mathrm{d}r \right) \, \Big| \, \mathcal{F}_s \right] \exp\left(- \int_{t}^{s} f(X_r,r)\,\mathrm{d}r \right) \,\Big|\, \mathcal{F}_t\right].
\end{aligned}
\end{equation*}
Substituting the definition of the value function at time $s$ from \eqref{eq:FK_rep} into the inner expectation, we obtain the first equality of the proposition:
\begin{equation*}
\exp \left( - V(X_t,t) \right) = \mathbb{E}_{\mathbb{P}^{0}}\left[\exp\left( - V(X_s, s) \right) \exp\left(- \int_{t}^{s} f(X_r,r)\,\mathrm{d}r \right) \,\Big|\, \mathcal{F}_t\right].
\end{equation*}

To derive the second equality, we perform a change of measure from the base distribution $p^{0}$ to the controlled distribution $\mathbb{P}^u$. By applying the Girsanov Theorem, the Radon-Nikodym derivative (likelihood ratio) over the interval $[t, s]$ is given by:
\begin{equation*}
\frac{\mathrm{d}\mathbb{P}^{0}}{\mathrm{d}\mathbb{P}^u} = \exp\left( - \int_t^s u(X_r, r) \mathrm{d}B_r - \frac{1}{2} \int_t^s \|u(X_r, r)\|^2 \mathrm{d}r \right).
\end{equation*}
Substituting this into the expectation, the second form of the recursive structure follows immediately.
\end{proof}
\section{Convergence Analysis of Recursive Update}
\label{appendix:convergence_analysis_of_recursive_update}
In this section, we establish the theoretical convergence of the iterative sequence $\{V_k\}_{k=1}^\infty$ to the true optimal value function $V$.
The proof leverages the contraction mapping principle in the space of transformed value functions.

\textbf{Restatement of Theorem \ref{thrm:convergence_of_bootstrapping}.}
\textit{Suppose the following assumptions hold:
\begin{enumerate}
\item[A1.] The running cost is bounded from below by a positive constant, i.e., $f(x, t) \geq f_{\inf} > 0$;
\item[A2.] The terminal cost $g(x)$ is continuous and bounded;
\item[A3.] The underlying SDE satisfies standard Lipschitz and growth conditions, ensuring the transition semigroup possesses the Feller property.
\end{enumerate}
Let $\mathcal{C}_b(\mathcal{X} \times [0, 1])$ be the Banach space of bounded continuous functions. Let $\{V_k\}_{k=1}^\infty \subset \mathcal{C}_b$ be a sequence of functions satisfying $V_k(x, 1) = g(x)$ for all $k$, and for $t < 1$, the recursive update:}
\begin{equation*}
\exp\left( - V_{k}(x, t) \right) = \mathbb{E}_{\mathbb{P}^{0}}\left[ \exp\left( - \tilde{V}_{k - 1}(X_{t+\varepsilon}, t+\varepsilon) - \int_{t}^{t+\varepsilon} f(X_r,r)\,\mathrm{d}r \right) \, \Bigg | \, X_t = x \right],
\end{equation*}
\textit{where $\tilde{V}_{k-1}$ is defined such that $\tilde{V}_{k-1}(x, s) = V_{k-1}(x, s)$ for $s < 1$ and $\tilde{V}_{k-1}(x, 1) = g(x)$. Then, the sequence $V_k(x, t)$ converges to the unique optimal value function $V(x, t)$ as $k \to \infty$.}

\begin{proof}
We define the transformed function $Z_k(x, t) = \exp(-V_k(x, t))$. The terminal condition implies $Z_k(x, 1) = \exp(-g(x)) \eqqcolon \zeta(x)$ for all $k$. We define the operator $\mathcal{T}$ on $\mathcal{C}_b$ as follows:
$$(\mathcal{T} \phi)(x, t) \coloneqq \mathbb{E} \left[ \tilde{\phi}(X_{t+\varepsilon}, t+\varepsilon) \exp\left( - \int_t^{t+\varepsilon} f(X_r, r) \mathrm{d}r \right) \Big| X_t = x \right],$$
where $\tilde{\phi}(x, s) = \phi(x, s)$ for $s < 1$ and $\tilde{\phi}(x, 1) = \zeta(x)$. Under Assumption A3, $\mathcal{T}$ maps $\mathcal{C}_b$ into itself.

For any $\phi, \psi \in \mathcal{C}_b$, consider the difference $|\mathcal{T}\phi - \mathcal{T}\psi|$. Note that at the boundary $s=1$, we have $\tilde{\phi}(x, 1) - \tilde{\psi}(x, 1) = \zeta(x) - \zeta(x) = 0$. For $t < 1-\varepsilon$:
$$|(\mathcal{T}\phi)(x, t) - (\mathcal{T}\psi)(x, t)| \leq \mathbb{E} \left[ |\phi(X_{t+\varepsilon}, t+\varepsilon) - \psi(X_{t+\varepsilon}, t+\varepsilon)| e^{-\int_t^{t+\varepsilon} f \mathrm{d}r} \right].$$
By Assumption A1, $e^{-\int_t^{t+\varepsilon} f \mathrm{d}r} \leq e^{-f_{\inf} \varepsilon} = \gamma < 1$. Thus, we obtain:
$$\| \mathcal{T}\phi - \mathcal{T}\psi \|_\infty \leq \gamma \| \phi - \psi \|_\infty.$$
This proves $\mathcal{T}$ is a contraction mapping. By the Banach Fixed Point Theorem, there exists a unique fixed point $Z \in \mathcal{C}_b$. 

As $Z_k \to Z$, it follows that $V_k \to V$.
\end{proof}

\section{Variance Analysis}
\label{appendix:variance_analysis}
In this section, we provide the theoretical justification for the variance reduction in our algorithm.
We first derive a general variance decomposition formula and then establish an explicit upper bound for the conditional variance under control discrepancies.

\textbf{Restatement of Proposition \ref{prop:variance_decomposition}.} \textit{Let $Z_t$ be the random variable $\exp(-g(X_1) - \int_t^1 f(X_r, r) \,\mathrm{d}r)$.
For any $t < s < 1$, the conditional variance $\mathbb{V}(Z_t|\mathcal{F}_t)$ can be decomposed as follows:}
\begin{equation*}
\begin{aligned}
&\mathbb{V}(Z_t|\mathcal{F}_t) \\ &= \mathbb{V}\left(\mathbb{E}[Z_s|\mathcal{F}_s] \exp\left(-\int_{t}^{s} f(X_r,r),\mathrm{d}r \right) \Bigg| \mathcal{F}_t\right) + \mathbb{E} \left[\mathbb{V}(Z_s|\mathcal{F}_s) \exp\left(- 2\int_{t}^{s} f(X_r,r),\mathrm{d}r \right) \Bigg| \mathcal{F}_t \right].
\end{aligned}
\end{equation*}
\begin{proof}
Consider the relationship between $Z_t$ and $Z_s$. By the definition of the integral, we have:
\begin{equation}
\label{eq:Z_relation}
Z_t = Z_s \exp\left( - \int_t^s f(X_r, r) ,\mathrm{d}r \right).
\end{equation}
To simplify notation, let $\Gamma_{t,s} = \exp\left( - \int_t^s f(X_r, r) \,\mathrm{d}r \right)$. 
Note that $\Gamma_{t,s}$ is $\mathcal{F}_s$-measurable but not $\mathcal{F}_t$-measurable. 
We apply the Law of Total Variance conditioned on $\mathcal{F}_s$ within the outer filtration $\mathcal{F}_t$:
\begin{equation}
\label{eq:total_var_law}
\mathbb{V}(Z_t | \mathcal{F}_t) = \mathbb{V}\big( \mathbb{E}[Z_t | \mathcal{F}_s] \big| \mathcal{F}_t \big) + \mathbb{E} \big[ \mathbb{V}(Z_t | \mathcal{F}_s) \big| \mathcal{F}_t \big].
\end{equation}
We analyze the two terms on the right-hand side of \eqref{eq:total_var_law} individually.

Using the relation \eqref{eq:Z_relation} and the fact that $\Gamma_{t,s}$ is $\mathcal{F}_s$-measurable, the inner expectation of the first term is:
\begin{equation*}
\mathbb{E}[Z_t | \mathcal{F}_s] = \mathbb{E}[Z_s \Gamma_{t,s} | \mathcal{F}_s] = \Gamma_{t,s} \mathbb{E}[Z_s | \mathcal{F}_s].
\end{equation*}
Taking the variance of this expression conditioned on $\mathcal{F}_t$ gives:
\begin{equation*}
\mathbb{V}\big( \mathbb{E}[Z_t | \mathcal{F}_s] \big| \mathcal{F}_t \big) = \mathbb{V}\left( \mathbb{E}[Z_s | \mathcal{F}_s] \exp\left( - \int_t^s f(X_r, r) ,\mathrm{d}r \right) \Bigg| \mathcal{F}_t \right).
\end{equation*}
This matches the first term of the proposition.

Similarly, we compute the inner variance in the second term:
\begin{equation*}
\mathbb{V}(Z_t | \mathcal{F}_s) = \mathbb{V}(Z_s \Gamma_{t,s} | \mathcal{F}_s).
\end{equation*}
Since $\Gamma_{t,s}$ is $\mathcal{F}_s$-measurable, it behaves as a constant with respect to the conditional variance $\mathbb{V}(\cdot | \mathcal{F}_s)$.
Thus, we have:
\begin{equation*}
\mathbb{V}(Z_t | \mathcal{F}_s) = \Gamma_{t,s}^2 \mathbb{V}(Z_s | \mathcal{F}_s) = \exp\left( - 2 \int_t^s f(X_r, r) ,\mathrm{d}r \right) \mathbb{V}(Z_s | \mathcal{F}_s).
\end{equation*}
Taking the expectation of this term conditioned on $\mathcal{F}_t$ yields:
\begin{equation*}
\mathbb{E} \big[ \mathbb{V}(Z_t | \mathcal{F}_s) \big| \mathcal{F}_t \big] = \mathbb{E} \left[ \mathbb{V}(Z_s | \mathcal{F}_s) \exp\left( - 2 \int_t^s f(X_r, r) ,\mathrm{d}r \right) \Bigg| \mathcal{F}_t \right].
\end{equation*}
This matches the second term of the proposition.
Combining the first and second terms completes the proof.
\end{proof}

\bigskip

\textbf{Restatement of Theorem \ref{thrm:var_bound}.} \textit{Suppose that the control $u$ satisfies the uniform bound $\sup_{x,r} \| u^*(x, r) - u(x, r) \|^2 \leq \kappa$. Then, the conditional variance of $\mathcal{M}(X_{[t,s]})$ (defined by Equation \ref{eq:M_definition}) under the controlled measure $\mathbb{P}^u$ satisfies:}
\begin{equation*}
\mathbb{V}_{\mathbb{P}^u}(\mathcal{M}(X_{[t,s]})|\mathcal{F}_t) \leq \exp\left(-2V(X_t,t) \right) \left[ \exp\left( \kappa(s-t) \right) - 1 \right].
\end{equation*}

\begin{proof}
Consider the log-estimator defined by $Y_{t,s} = \log \mathcal{M}(X_{[t,s]})$. By applying Itô's formula to the definition of $Y_{t,s}$, we obtain:
\begin{equation*}
\begin{aligned}
\mathrm{d}Y_{t,s} = -\left[ \frac{\partial V}{\partial s} + \nabla V \cdot b + \nabla V \cdot (\sigma u) + \frac{1}{2} \text{tr}(\sigma \sigma^\top \nabla^2 V) + f + \frac{1}{2}\|u\|^2 \right] \mathrm{d}s - (\sigma^\top \nabla V + u) \cdot \mathrm{d}B_s.
\end{aligned}
\end{equation*}
Recall that in our framework, the value function $V$ satisfies the Hamilton-Jacobi-Bellman (HJB) equation:
\begin{equation}\label{eq:hjb_ref}
\frac{\partial V}{\partial s} + \nabla V \cdot b + \frac{1}{2} \text{tr}(\sigma \sigma^\top \nabla^2 V) + f - \frac{1}{2}\|\sigma^\top \nabla V\|^2 = 0.
\end{equation}
By substituting \eqref{eq:hjb_ref} into the expression for $\mathrm{d}Y_{t,s}$ and invoking the optimal control relation $u^* = - \sigma^\top\nabla V$, the dynamics simplify to:
\begin{equation*}
\begin{aligned}
\mathrm{d}Y_{t,s} &= -\left[ \frac{1}{2}\|\sigma^\top \nabla V\|^2 + \nabla V \cdot (\sigma u) + \frac{1}{2}\|u\|^2 \right] \mathrm{d}s - (\sigma^\top \nabla V + u) \cdot \mathrm{d}B_s \\
&= - \frac{1}{2} \|u - u^*\|^2 \mathrm{d}s - (u - u^*) \cdot \mathrm{d} B_s.
\end{aligned}
\end{equation*}
Integrating from $t$ to $s$ with the initial condition $Y_{t,t}=-V(X_t,t)$, we have:
\begin{equation*}
\begin{aligned}
Y_{t,s} = -V(X_t,t) - \int_t^s \frac{1}{2} \|u - u^*\|^2 \mathrm{d}r - \int_t^s (u - u^*) \cdot \mathrm{d} B_r.
\end{aligned}
\end{equation*}

Next, we apply Itô's formula to the estimator $\mathcal{M}(X_{[t,s]}) = \exp(Y_{t,s})$. The evolution of $\mathcal{M}(X_{[t,s]})$ is governed by:
\begin{equation}\label{eq:M_dynamics}
\begin{aligned}
\mathrm{d} \mathcal{M}(X_{[t,s]}) &= \mathcal{M}(X_{[t,s]}) \, \mathrm{d} Y_{t,s} + \frac{1}{2} \mathcal{M}(X_{[t,s]}) \left( \mathrm{d} Y_{t,s} \right)^2 \\
&= \mathcal{M}(X_{[t,s]}) \left( u^* - u \right) \cdot \mathrm{d} B_s,
\end{aligned}
\end{equation}
which yields the integral representation:
\begin{equation*}
\begin{aligned}
\mathcal{M}(X_{[t,s]}) = \exp \left(- V(X_t,t) \right) + \int_t^s \mathcal{M}(X_{[t,r]}) \left( u^* - u \right) \cdot \mathrm{d} B_r.
\end{aligned}
\end{equation*}

By applying Itô's isometry to \eqref{eq:M_dynamics}, the conditional variance is bounded as follows:
\begin{equation}\label{eq:var_integral_bound}
\begin{aligned}
\mathbb{V}_{\mathbb{P}^u}(\mathcal{M}(X_{[t,s]})|\mathcal{F}_t) &= \mathbb{E}_{\mathbb{P}^u} \left[\int_t^s \mathcal{M}(X_{[t,r]})^2 \left \| u^* - u \right \|^2 \mathrm{d} r \, \Big| \mathcal{F}_t \right] \\
&\leq \kappa \int_t^s \mathbb{E}_{\mathbb{P}^u} \left[ \mathcal{M}(X_{[t,r]})^2 \, \Big| \mathcal{F}_t\right] \mathrm{d}r.
\end{aligned}
\end{equation}

To estimate the expectation in \eqref{eq:var_integral_bound}, we introduce a measure transformation. Consider an auxiliary process under distribution $\mathbb{Q}^u$ defined by the shifted SDE $\mathrm{d}X_r = (b - \sigma u + 2\sigma u^*) \mathrm{d}r + \sigma \mathrm{d}B_r$. According to Girsanov's Theorem:
\begin{equation*}
\begin{aligned}
&\mathbb{E}_{\mathbb{P}^u} \left[ \mathcal{M}(X_{[t,r]})^2 \, \Big| \mathcal{F}_t\right] \\
&= \mathbb{E}_{\mathbb{P}^u} \Bigg[ \exp \left( -2V - \int_t^r \|u - u^*\|^2 \mathrm{d}\tau - 2\int_t^r (u - u^*) \cdot \mathrm{d} B_\tau \right)  \, \Big| \mathcal{F}_t \Bigg] \\
&= \mathbb{E}_{\mathbb{P}^u} \Bigg[ \exp \left( -2V + \int_t^r \|u - u^*\|^2 \mathrm{d}\tau \right) \exp \left(- 2\int_t^r \|u - u^*\|^2 \mathrm{d}\tau - 2\int_t^r (u - u^*) \cdot \mathrm{d} B_\tau \right)   \, \Big| \mathcal{F}_t \Bigg] \\
&= \mathbb{E}_{\mathbb{P}^u} \Bigg[ \exp \left( -2V + \int_t^r \|u - u^*\|^2 \mathrm{d}\tau \right) \frac{\mathrm{d}\mathbb{Q}^u}{\mathrm{d}\mathbb{P}^u} \, \Big| \mathcal{F}_t \Bigg] \\
&= \mathbb{E}_{\mathbb{Q}^u} \Bigg[ \exp \left( -2V + \int_t^r \|u - u^*\|^2 \mathrm{d}\tau \right) \, \Big| \mathcal{F}_t \Bigg] \\
&\leq \exp \left( -2V(X_t,t) \right) \exp \left( \kappa(r-t) \right).
\end{aligned}
\end{equation*}
Substituting this bound back into \eqref{eq:var_integral_bound} and performing the final integration, we conclude:
\begin{equation*}
\begin{aligned}
\mathbb{V}_{p^u}(\mathcal{M}(X_{[t,s]})|\mathcal{F}_t) &\leq \kappa \int_t^s \exp \left( -2V(X_t,t) \right) \exp \left( \kappa(r-t) \right) \mathrm{d}r \\
&= \exp \left( -2V(X_t,t) \right) \left[ \exp\left( \kappa(s-t) \right) - 1 \right].
\end{aligned}
\end{equation*}
\end{proof}

\section{Further Ablation Studies}
\subsection{Ablation study on neural network depths}
We performed an ablation study on the OU Linear task to investigate the impact of neural network depth on control performance. Three network architectures were evaluated: Shallow (64-128-64), Default (64-128-256-128-64), and Deep (64-128-256-512-256-128-64). The control $\mathcal{L}_2$ errors over training iterations are reported in Table~\cref{tab:network_depth}.
\begin{table}[h]
\label{tab:network_depth}
\centering
\caption{Ablation study on neural network depths over iterations}
\begin{tabular}{c|c|c|c}
\hline
Iterations & Shallow & Default & Deep \\
\hline
3000 & 0.000818 & 0.000542 & 0.000434 \\
5000 & 0.000289 & 0.000210 & 0.000192 \\
7000 & 0.000157 & 0.000125 & 0.000123 \\
9000 & 0.000094 & \textbf{0.000082} & 0.000089 \\
\hline
\end{tabular}
\end{table}
The results show that deeper networks consistently achieve lower $\mathcal{L}_2$ control errors compared to shallower architectures, indicating improved function approximation capability. Notably, the Deep network achieves the fastest reduction in error during early iterations, suggesting that additional layers help capture the dynamics of the OU Linear task more effectively. However, the performance gain from Deep over Default becomes marginal after 7000 iterations, implying that a moderately deep network already provides sufficient capacity for this task. Overall, this study highlights the trade-off between network complexity and computational efficiency when selecting architectures for PI-VM training.

\subsection{Ablation study on batch size $B$ and sample size $N$}
We conducted an ablation study to evaluate the impact of batch size $B$ and the number of trajectory samples $N$ on the OU Linear task. Table~\cref{tab:different_batch} reports the control $\mathcal{L}_2$ errors over training iterations for three different $(B, N)$ configurations.
\begin{table}[h]
\label{tab:different_batch}
\centering
\caption{Ablation study on batch size $B$ and sample size $N$}
\begin{tabular}{c|c|c|c}
\hline
Iterations & B=6400, N=8 & B=3200, N=16 & B=1600, N=32 \\
\hline
3000 & 0.000542 & 0.000571 & 0.000722 \\
5000 & 0.000210 & 0.000244 & 0.000299 \\
9000 & \textbf{0.000082} & 0.000093 & 0.000114 \\
\hline
\end{tabular}
\end{table}
Larger batch sizes combined with smaller sample numbers (B=6400, N=8) consistently achieve the lowest control errors, suggesting that a sufficiently large batch stabilizes the TD updates while moderate sampling per state is sufficient to reduce variance. Reducing the batch size or increasing the number of samples per state slightly degrades performance, indicating a trade-off between the number of parallel trajectories and per-state variance reduction.

\subsection{Ablation study on sample size $N$ (fixed batch $B$)}
To further isolate the effect of sample size $N$, we fixed the batch size at $B=6400$ and varied $N$. Table~\cref{tab:fixed_batch} summarizes the control $\mathcal{L}_2$ errors for different sample sizes.
\begin{table}[h]
\label{tab:fixed_batch}
\centering
\caption{Ablation study on fixed batch size $B$ with varying sample size $N$}
\begin{tabular}{c|c|c|c}
\hline
Iterations & B=6400, N=8 & B=6400, N=16 & B=6400, N=32 \\
\hline
3000 & 0.000542 & 0.000483 & 0.000477 \\
9000 & {0.000082} & {0.000082} & \textbf{0.000081} \\
\hline
\end{tabular}
\end{table}
The results show that increasing the number of samples $N$ per state slightly improves the final control performance, with diminishing returns after $N=16$. This suggests that, with a large enough batch size, moderate per-state sampling is sufficient to obtain stable TD targets, and additional samples yield only marginal gains while increasing computational cost.


\section{Further Related Works}
\paragraph{Policy-based SOC Solvers.}
Existing SOC solvers predominantly follow the policy-based paradigm, often termed Iterative Diffusion Optimization. These methods directly parameterize the control policy and optimize it by minimizing divergences (e.g., Relative Entropy, Cross-Entropy~\cite{holdijk2023stochastic}, Variance~\cite{nusken2021solving}, and Log-Variance~\cite{nusken2021solving}) between the optimal path measures and those induced by the controlled SDEs. Conversely, a parallel line of research investigates matching-based strategies, including Stochastic Optimal Control Matching~\cite{domingo2024stochastic} and Adjoint Matching~\cite{domingo-enrich2025adjoint}. A detailed summary and discussion of above methods can be found in~\cite{domingo2024taxonomy}. To improve training stability, recent studies have incorporated trust-region techniques within the path space~\cite{blessing2025trust, von2025learning}. Beyond standard control problems, these techniques are finding increasing utility in sampling tasks~\cite{liu2025adjoint, havens2025adjoint} and diffusion model fine-tuning~\cite{domingo-enrich2025adjoint}. Their training pipelines rely heavily on forward simulation and stochastic back-propagation. However, this on-policy simulation mechanism suffers from significant drawbacks: it is not only computationally expensive but also prone to trapping the policy in local optima. Moreover, the high variance associated with gradient estimation often leads to instability during the training process.

\paragraph{Value-based SOC Solvers.} Instead of directly finding the optimal control in the non-convex landscape, HJB points out a different way. Many works have been proposed try to solve HJB numerically, such as finite difference methods~\cite{bonnans2004fast, bavnas2022numerical}, finite element methods~\cite{jensen2013convergence} and semi-Lagrangian scheme~\cite{debrabant2013semi, carlini2020semi}. However, such discretize-then-optimize paradigm suffer from the curse of dimensionality, besides, the high-order term involed in the HJB equation further hinder its scalability. 

Compared to the direct pde solver, the celebrated path integral control~\cite{todorov2006linearly, kappen2007introduction, kappen2005linear} follows another paradigm. It follows the Cole-Hopf transform~\cite{evans2022partial} and Feynman-Kac Lemma~\cite{oksendal2013stochastic}, convert the original non-trival PDE problem into a sampling problem.  Based on this theory, many algorithms and improvements have been proposed such as Policy Improvement~\cite{theodorou2010generalized}, Model Predictive~\cite{williams2016aggressive, williams2017model}, State Feedback~\cite{thijssen2015path, gomez2014policy} and Cross Entropy~\cite{kappen2016adaptive, zhang2014applications}. While these methods have demonstrated success in some low-dimensional tasks such as driving~\cite{williams2016aggressive} and trajectory optimization~\cite{kazim2024recent}, a critical limitation persists: the variance of the path integral estimator scales poorly with dimensionality. In high-dimensional settings, the Monte Carlo estimates suffer from excessive variance, leading to prohibitive sample complexity and unstable convergence. For full introduction and survey, please refer to~\cite{kazim2024recent}.

\paragraph{Continuous time RL.}


The theoretical foundation of continuous-time reinforcement learning was fundamentally established by~\cite{wang2020reinforcement} through the development of an exploratory Hamilton-Jacobi-Bellman equation that incorporates entropy regularization.
This foundational framework has been further extended to model-free temporal difference methods by~\cite{settai2025temporal} and to complex systems involving jump-diffusion processes by~\cite{cheridito2025deep}.
In a parallel development,~\cite{quer2024connecting} identified a mathematical equivalence between variance minimization in importance sampling and LQ-SOC.
Another significant research trajectory emphasizes discretization-free methodologies by leveraging the martingale characterization of value functions and the $q$-learning theory introduced by~\cite{jia2022policy,jia2023q}.
This line of work has most recently been synthesized with diffusion-guided policies by~\cite{hua2025continuous} to achieve robust and scalable reinforcement learning for high-dimensional continuous control tasks.
Existing methods either solve the HJB equation directly or rely on standard RL policy iteration.
Conversely, we leverage the analytical structure of the optimal value function in LQ-SOC, enabling the direct learning of the value function and the immediate recovery of optimal control via its gradient.




\section{Full Algorithm}
We provide the full algorithm pipeline in \cref{alg:pi_td_full}.

\begin{algorithm*}[t]
   \caption{Path Integral Value Matching (PI-VM)}
   \label{alg:pi_td_full}
   \begin{algorithmic}[1] 
      \State \textbf{Input:} Problem Setup $\{f,g,b,\sigma,\rho_0\}$; Sample Steps $M$; Sample Batch $N$; Replay Buffer Update Batch Size $K$; Training Batch Size $B$; Learning Rate $\eta$; EMA Param $\tau$.
      \State \textbf{Initialize:} Replay Buffer $\mathcal{D}$; Value Network $V_\theta$; Target Value Network ${V}_{\hat{\theta}}$.
      
      \For{episode $e = 1$ \textbf{to} max episodes}
         \If{need update replay buffer}
            \State Set control $u_t(x) = -\sigma^T(t)\nabla {V}_{\hat{\theta}}(x,t)$.
            \State Sample $K$ trajectories $\{X^{(i)}\}_{i=1}^K$.
            \For{each state-time pair $\{ X^{(i)}_t, t\}$ in trajectories}
               \State Sample $N$ short-term path branches $\{Y_{[t,s]}^{(j)} \mid Y_{t}^{(j)} = X^{(i)}_t \}_{j=1}^N$, where $s = \min(1, t + M\Delta t)$.
               \State Compute functionals $\{\hat{\mathcal{W}}^{(j)}, \hat{\mathcal{S}}^{(j)}\}_{j=1}^N$ via \cref{eq:functional_1,eq:functional_2}.
               \State Push $\{X^{(i)}_t, t, s, \{Y_{s}^{(j)}, \hat{\mathcal{W}}^{(j)}, \hat{\mathcal{S}}^{(j)}\}_{j=1}^N\}$ into $\mathcal{D}$.
            \EndFor            
         \EndIf

         \State Sample batch $\{x_{t_i}, t_i, s_i, \{y^{(j)}_{s_i}, \hat{\mathcal{W}}^{(j)}, \hat{\mathcal{S}}^{(j)}\}_{j=1}^N\}_{i=1}^B \sim \mathcal{D}$.
         \State Compute terminal functionals $\{\{\hat{\mathcal{G}}(y^{(j)}_{s_i})\}_{j=1}^N\}_{i=1}^B$ following \cref{eq:functional_1}.
         \State Compute target values $\{ \hat{V}_{\hat{\theta}}(x_{t_i}, u_{t_i}, t_i, s_i)\}_{i=1}^B$ following \cref{eq:target_value_controlled}.
         \State Compute Loss: $\mathcal{L}(\theta) = \frac{1}{B} \sum_{i=1}^B \| V_\theta(x_{t_i}, t_i) - \hat{V}_{\hat{\theta}}(x_{t_i}, u_{t_i}, t_i, s_i) \|^2$.
         \State Update parameters: $\theta \leftarrow \theta - \eta \nabla_\theta \mathcal{L}(\theta)$.
         \State Update target parameters: $\hat{\theta} \leftarrow \tau {\theta} + (1-\tau)\hat{\theta}$.
      \EndFor
      \State \textbf{Output:} Optimal Value Function $V_\theta$.
   \end{algorithmic}
\end{algorithm*}

\section{SOC Tasks Settings}
\label{appendix:soc_exp_settings}

\subsection{Problem 1: Quadratic Ornstein-Uhlenbeck}

\subsubsection{Dynamics \& Cost Functions}
\begin{itemize}
    \item \textbf{Source Distribution ($\rho_0(x)$):} The process initializes from a scaled standard Gaussian:
    \begin{equation}
        x_0 \sim \mathcal{N}(0, 0.5^2 I) \implies \Sigma_{\text{init}} = 0.25 I
    \end{equation}
    
    \item \textbf{Base Drift ($b$) \& Gradient ($\nabla_x b$):} Linear drift dynamics:
    \begin{equation}
        b(x, t) = Ax, \quad \nabla_x b(x, t) = A
    \end{equation}
    
    \item \textbf{Diffusion Coefficient ($\sigma$):} Constant diffusion:
    \begin{equation}
        \sigma(t) = \sigma_0
    \end{equation}
    
    \item \textbf{Running Cost ($f$) \& Gradient ($\nabla_x f$):} Quadratic running cost (assuming symmetric $P$):
    \begin{equation}
        f(x, t) = x^\top P x, \quad \nabla_x f(x, t) = 2Px
    \end{equation}
    
    \item \textbf{Terminal Cost ($g$) \& Gradient ($\nabla_x g$):} Quadratic terminal cost (assuming symmetric $Q$):
    \begin{equation}
        g(x) = x^\top Q x, \quad \nabla_x g(x) = 2Qx
    \end{equation}
\end{itemize}

\subsubsection{Analytical Optimal Control $u^*$}
The optimal control is linear in the state $x$, derived from the solution to the Riccati equation.
\begin{equation}
    u^*_t(x) = -2\sigma_0^\top F_t x
\end{equation}
where $F_t$ is the solution to the following \textbf{Riccati Differential Equation} (solved backward in time):
\begin{equation}
    \frac{dF_t}{dt} + A^\top F_t + F_t A - 2 F_t \sigma_0 \sigma_0^\top F_t + P = 0, \quad \text{with } F_T = Q
\end{equation}

\subsubsection{Analytical Optimal Value Function}
The form of the optimal control implies that the Value Function is quadratic with respect to the state $x$:
\begin{equation}
    V(x, t) = x^\top F_t x + \alpha(t)
\end{equation}
where:
\begin{itemize}
    \item $F_t$ (Hessian of Value) is the solution to the Riccati differential equation mentioned above.
    \item $\alpha(t)$ (Scalar Offset) accounts for the constant term evolution caused by stochastic noise:
    \begin{equation}
        \alpha(t) = \int_t^T \text{Tr}(\sigma_0 \sigma_0^\top F_s) \, ds
    \end{equation}
\end{itemize}

\subsubsection{Specific Parameter Settings}
Two difficulty settings are defined: Easy and Hard.

\begin{table}[h]
    \centering
    \caption{Parameter Settings for Quadratic OU}
    \label{tab:quad_ou_params}
    \begin{tabular}{lcc}
    \toprule
    \textbf{Parameter} & \textbf{Easy Setting} & \textbf{Hard Setting} \\
    \midrule
    Dimension ($d$) & 20 & 20 \\
    Drift Matrix ($A$) & $0.2 \cdot I$ & $I$ \\
    Running Cost ($P$) & $0.2 \cdot I$ & $I$ \\
    Terminal Cost ($Q$) & $0.1 \cdot I$ & $0.5 \cdot I$ \\
    Diffusion ($\sigma_0$) & $I$ & $I$ \\
    Time Steps ($T$) & 50 & 50 \\
    Learning Rates ($\eta$) & 1e-4 & 1e-4 \\
    Batch Size ($B$) & 6400 & 12800 \\
    Max Episodes & 60000 & 80000 \\
    \bottomrule
    \end{tabular}
\end{table}

\subsection{Problem 2: Linear Ornstein-Uhlenbeck}

\subsubsection{Dynamics \& Cost Functions}
\begin{itemize}
    \item \textbf{Source Distribution ($\rho_0(x)$):} The process initializes from a scaled standard Gaussian:
    \begin{equation}
        x_0 \sim \mathcal{N}(0, 0.5^2 I)
    \end{equation}
    
    \item \textbf{Base Drift ($b$) \& Gradient ($\nabla_x b$):} Linear drift:
    \begin{equation}
        b(x, t) = Ax, \quad \nabla_x b(x, t) = A
    \end{equation}
    
    \item \textbf{Diffusion Coefficient ($\sigma$):} Constant diffusion (matrix valued):
    \begin{equation}
        \sigma(t) = \sigma_0
    \end{equation}
    
    \item \textbf{Running Cost ($f$) \& Gradient ($\nabla_x f$):} Zero running cost:
    \begin{equation}
        f(x, t) = 0, \quad \nabla_x f(x, t) = 0
    \end{equation}
    
    \item \textbf{Terminal Cost ($g$) \& Gradient ($\nabla_x g$):} Linear terminal cost (inner product):
    \begin{equation}
        g(x) = \langle \gamma, x \rangle = \gamma^\top x, \quad \nabla_x g(x) = \gamma
    \end{equation}
\end{itemize}

\subsubsection{Analytical Optimal Control $u^*$}
For linear costs, the optimal control depends only on time (open-loop component). It is derived analytically as:
\begin{equation}
    u^*_t(x) = -\sigma_0^\top e^{A^\top (T-t)} \gamma
\end{equation}

\subsubsection{Specific Parameter Settings}
Parameters are randomly sampled once at the beginning of the simulation using a fixed seed.

\begin{itemize}
    \item \textbf{Dimension ($d$):} 10
    \item \textbf{Random Sampling ($\xi_{ij}$):} Sampled from a standard normal distribution (implied by typical setup, e.g., $0.1 \times \mathcal{N}(0,1)$).
    \item \textbf{Drift Matrix ($A$):} $A = -I + (\xi_{ij})_{1 \le i,j \le d}$
    \item \textbf{Diffusion ($\sigma_0$):} $\sigma_0 = I + (\xi_{ij})_{1 \le i,j \le d}$
    \item \textbf{Cost Vector ($\gamma$):} $\gamma = \mathbf{1}$ (Vector of all ones)
    \item \textbf{Time Steps ($T$):} 100
    \item \textbf{Learning Rates ($\eta$):} 1e-4
    \item \textbf{Batch Size ($B$):} 6400
    \item \textbf{Max Episodes:} 60000
\end{itemize}

\subsubsection{Analytical Optimal Value Function}
Since the cost function and dynamics are linear, the Value Function is a linear function of $x$ plus a time-dependent scalar offset:
\begin{equation}
    V(x, t) = \psi(t)^\top x + \beta(t)
\end{equation}
where:
\begin{itemize}
    \item $\psi(t)$ (Gradient of Value) satisfies $\frac{d\psi}{dt} = -A^\top \psi$ with $\psi(T) = \gamma$. The analytical solution is:
    \begin{equation}
        \psi(t) = e^{A^\top (T-t)} \gamma
    \end{equation}
    \item $\beta(t)$ (Scalar Offset) is given by:
    \begin{equation}
        \beta(t) = -\frac{1}{2} \int_t^T \psi(s)^\top \sigma_0 \sigma_0^\top \psi(s) \, ds
    \end{equation}
\end{itemize}

\section{Sampling Tasks Settings}
\label{appendix:sampling_exp_settings}

\subsection{Shared Dynamics \& Notation}
Both the GMM problem and the Many Well problem utilize a controlled Ornstein-Uhlenbeck (OU) process with the following shared settings:
\begin{itemize}
    \item \textbf{State Space:} $x \in \mathbb{R}^d$
    \item \textbf{Time Horizon:} $t \in [0, 1]$
    \item \textbf{Diffusivity Schedule:} A cosine schedule is used to modulate the noise level:
    \begin{equation}
        \zeta(t) = (C_{\max} - C_{\min}) \cos^2 \left( \frac{t\pi}{2T} \right) + C_{\min}
    \end{equation}
    \item \textbf{Noise Coefficient:} Defined as a function of the diffusivity:
    \begin{equation}
        \sigma(t) = \eta \sqrt{2\zeta(t)}
    \end{equation}
\end{itemize}

\subsection{Problem 3: Gaussian Mixture Model (GMM)}

\subsubsection{Dynamics \& Cost Functions}

\begin{itemize}
    \item \textbf{Source Distribution ($\rho_0$):} The process initializes from a wide Gaussian distribution with $\eta = 2.5$:
    \begin{equation}
        \rho_0(x) = \mathcal{N}(x \mid 0, \eta^2 I)
    \end{equation}

    \item \textbf{Base Drift ($b$) \& Gradient ($\nabla_x b$):} Identical to the Many Well problem, using a linear mean-reverting drift:
    \begin{equation}
        b(x, t) = -\zeta(t) x, \quad \nabla_x b(x, t) = -\zeta(t) I
    \end{equation}

    \item \textbf{Running Cost ($f$) \& Gradient ($\nabla_x f$):} Zero running cost:
    \begin{equation}
        f(x, t) = 0, \quad \nabla_x f(x, t) = 0
    \end{equation}

    \item \textbf{Terminal Cost ($g$) \& Gradient ($\nabla_x g$):}
    The terminal cost forces the terminal distribution to match the target GMM $\rho_{\text{target}}(x) = \sum_{k=1}^K \pi_k \mathcal{N}(x \mid \mu_k, \Sigma_k)$.
    \begin{equation}
        g(x) = \log \mathbb{P}_T(x) - \log \rho_{\text{target}}(x)
    \end{equation}
    Given the prior stationary distribution $\mathbb{P}_T(x) = \mathcal{N}(0, \eta^2 I)$, the explicit cost is:
    \begin{equation}
    \begin{aligned}
        g(x) = &- \left( \frac{D}{2}\log(2\pi\eta^2) + \frac{\|x\|^2}{2\eta^2} \right) \\
               &- \log \left( \sum_{k=1}^K \pi_k \exp\left( -\frac{1}{2}(x-\mu_k)^\top \Sigma_k^{-1} (x-\mu_k) - \frac{1}{2}\log |2\pi \Sigma_k| \right) \right)
    \end{aligned}
    \end{equation}
    
    The gradient is derived using the posterior responsibility $\gamma_k(x)$:
    \begin{equation}
        \nabla_x g(x) = -\frac{x}{\eta^2} + \sum_{k=1}^K \gamma_k(x) \Sigma_k^{-1} (x - \mu_k)
    \end{equation}
    where $\gamma_k(x)$ is defined as:
    \begin{equation}
        \gamma_k(x) = \frac{\pi_k \mathcal{N}(x \mid \mu_k, \Sigma_k)}{\sum_{j=1}^K \pi_j \mathcal{N}(x \mid \mu_j, \Sigma_j)}
    \end{equation}
\end{itemize}

\subsubsection{Analytical Optimal Control $u^*$}
The optimal control is given by the score matching formula:
\begin{equation}
    u^*(x, t) = \eta \sqrt{2\zeta(t)} \left( \nabla_x \log \mathbb{Q}_t(x) - \nabla_x \log \mathbb{P}_t(x) \right)
\end{equation}

\textbf{1. Forward Score ($\nabla \log \mathbb{P}_t$):}
Since the prior dynamics preserve the stationary distribution $\mathcal{N}(0, \eta^2 I)$:
\begin{equation}
    \nabla_x \log \mathbb{P}_t(x) = -\frac{x}{\eta^2}
\end{equation}

\textbf{2. Backward Path Measure ($\mathbb{Q}_t$):}
The target distribution evolves as a time-varying GMM:
\begin{equation}
    \mathbb{Q}_t(x) = \sum_{k=1}^K \pi_k \mathcal{N}(x \mid \mu_k(t), \Sigma_k(t))
\end{equation}
The parameters evolve according to the scale factor $E(t) = \exp\left( -\int_t^T \zeta(s) \, ds \right)$. The explicit integral of the schedule is:
\begin{equation}
    \int_t^s \zeta(\tau) \, d\tau = \frac{C_{\max} + C_{\min}}{2} (s - t) + \frac{(C_{\max} - C_{\min}) T}{2\pi} \left[ \sin\left( \frac{\pi s}{T} \right) - \sin\left( \frac{\pi t}{T} \right) \right]
\end{equation}
The time-dependent means and covariances are:
\begin{align}
    \mu_k(t) &= \mu_k \cdot E(t) \\
    \Sigma_k(t) &= \Sigma_k E(t)^2 + \eta^2 (1 - E(t)^2) I
\end{align}

\textbf{3. Final Expression:}
Substituting these into the control formula:
\begin{equation}
    u^*(x, t) = \eta \sqrt{2\zeta(t)} \left( \left[ \sum_{k=1}^K \gamma_k(x, t) \Sigma_k(t)^{-1} (\mu_k(t) - x) \right] + \frac{x}{\eta^2} \right)
\end{equation}
where $\gamma_k(x, t)$ is the time-dependent responsibility:
\begin{equation}
    \gamma_k(x, t) = \frac{\pi_k \mathcal{N}(x \mid \mu_k(t), \Sigma_k(t))}{\sum_{j} \pi_j \mathcal{N}(x \mid \mu_j(t), \Sigma_j(t))}
\end{equation}

\subsubsection{Analytical Optimal Value Function}
The value function is the difference between the log-probabilities of the target path measure and the prior path measure:
\begin{equation}
    V(x, t) = - \left( \log \mathbb{Q}_t(x) - \log \mathbb{P}_t(x) \right)
\end{equation}
Expanding this yields:
\begin{equation}
\begin{aligned}
    V(x, t) = & \underbrace{- \left( \frac{D}{2}\log(2\pi\eta^2) + \frac{\|x\|^2}{2\eta^2} \right)}_{\text{Prior Log-Prob } (\log \mathbb{P}_t)} \\
    &- \underbrace{\log \left( \sum_{k=1}^K \pi_k \frac{1}{\sqrt{(2\pi)^D |\Sigma_k(t)|}} \exp\left( -\frac{1}{2}(x - \mu_k(t))^\top \Sigma_k(t)^{-1} (x - \mu_k(t)) \right) \right)}_{\text{Target Path Log-Prob } (\log \mathbb{Q}_t)}
\end{aligned}
\end{equation}

\subsubsection{Specific Parameter Settings}
Parameters are randomly sampled once at the beginning of the simulation using a fixed seed.

\begin{itemize}
    \item \textbf{Dimension ($d$):} 20
    \item \textbf{Time Steps ($T$):} 50
    \item \textbf{Learning Rates ($\eta$):} 5e-4
    \item \textbf{Batch Size ($B$):} 12800
    \item \textbf{GMM modes number:} 4
    \item \textbf{Max Episodes:} 30000
    \item \textbf{Dist:} Far $\mu \sim \mathcal{N}(0,4I)$, Close $\mu \sim \mathcal{N}(0,I)$
     \item \textbf{Var:} Small $\Sigma = 0.5 \cdot I$, Large $\Sigma =  I$  

\end{itemize}

\subsection{Problem 4: Many Well}

\subsubsection{Dynamics \& Cost Functions}

\begin{itemize}
    \item \textbf{Source Distribution ($\rho_0$):} The process initializes from a wide Gaussian distribution with $\eta = 1$:
    \begin{equation}
        \rho_0(x) = \mathcal{N}(x \mid 0, \eta^2 I)
    \end{equation}

    \item \textbf{Base Drift ($b$) \& Gradient ($\nabla_x b$):} Identical to the Many Well problem, using a linear mean-reverting drift:
    \begin{equation}
        b(x, t) = -\zeta(t) x, \quad \nabla_x b(x, t) = -\zeta(t) I
    \end{equation}

    \item \textbf{Running Cost ($f$) \& Gradient ($\nabla_x f$):} Zero running cost:
    \begin{equation}
        f(x, t) = 0, \quad \nabla_x f(x, t) = 0
    \end{equation}

    \item \textbf{Terminal Cost ($g$) \& Gradient ($\nabla_x g$):}
    The terminal cost forces the terminal distribution to match the target Many Well 
    Defined as $g(x) = \log \mathbb{P}_T(x) - \log \rho_{\text{target}}(x).$
    
    $$\mathbb{P}_T(x) \propto \exp\left(-\frac{\|x\|^2}{2\eta^2}\right)$$ 
    
    $$\rho_{\text{target}}(x) = \exp\left( - \sum_{i=1}^m (x_i^2 - \delta)^2 - \frac{1}{2} \sum_{j=m+1}^d x_j^2 \right)$$
    
Combining these (ignoring constants):

$$g(x) = \sum_{i=1}^m (x_i^2 - \delta)^2 + \frac{1}{2} \sum_{j=m+1}^d x_j^2 - \frac{1}{2\eta^2} \sum_{k=1}^d x_k^2$$
  
The gradient $\nabla_x g(x)$ is defined component-wise:

For $i \in \{1, \dots, m\}$ (Many Well dimensions):    $[\nabla_x g(x)]_i = 4x_i(x_i^2 - \delta) - \frac{x_i}{\eta^2}$

For $j \in \{m+1, \dots, d\}$ (Gaussian dimensions). $[\nabla_x g(x)]_j = x_j - \frac{x_j}{\eta^2}$
\end{itemize}

\subsubsection{Specific Parameter Settings}
Parameters are randomly sampled once at the beginning of the simulation using a fixed seed.

\begin{itemize}
    \item \textbf{Dimension ($d$):} 10
    \item \textbf{Time Steps ($T$):} 100
    \item \textbf{Learning Rates ($\eta$):} 5e-4
    \item \textbf{Batch Size ($B$):} 12800
    \item \textbf{Max Episodes:} 50000
    \item \textbf{Many Well modes ($m$):} 2
     \item \textbf{Many Well mean ($\delta$):} 2 

\end{itemize}





\end{document}